\documentclass[10pt,journal]{IEEEtran}

\ifCLASSOPTIONcompsoc
  \usepackage[nocompress]{cite}
\else
  \usepackage{cite}
\fi
\ifCLASSINFOpdf
\else
\fi

\usepackage[compact]{titlesec}
\usepackage{subfigure}
\usepackage{graphicx}

\usepackage{amsmath,amsthm}
\usepackage{amssymb,wasysym}
\usepackage{mathrsfs}
\usepackage{cite}
\usepackage{color}
\usepackage{url}

\usepackage{bbm}

\usepackage{wrapfig}
\usepackage{verbatim}
\usepackage{dsfont}
\usepackage{amsmath}
\usepackage{algorithm}
 \usepackage{algorithmic}
\usepackage{xurl}

\newtheorem{theorem}{\textbf{Theorem}}
\newtheorem{lemma}{\textbf{Lemma}}
\newtheorem{corollary}{\textbf{Corollary}}

\newtheorem{assumption}{\textbf{Assumption}}
\newtheorem{proposition}{\textbf{Proposition}}
\renewcommand{\algorithmicrequire}{\textbf{Input:}}  
\renewcommand{\algorithmicensure}{\textbf{Output:}} 
\allowdisplaybreaks[3]
\usepackage{tikz}

\ifodd 1
\definecolor{darkgreen}{RGB}{0,200,0}
\else

\fi

\begin{document}

%
\title{{Hierarchical Split Federated Learning: Convergence Analysis and System Optimization}}

%
%
%
%


\author{Zheng~Lin, Wei~Wei, Zhe~Chen,~\IEEEmembership{Member,~IEEE,} Chan-Tong Lam,~\IEEEmembership{Senior Member,~IEEE,}\\  Xianhao Chen,~\IEEEmembership{Member,~IEEE}, Yue Gao,~\IEEEmembership{Fellow,~IEEE}, and Jun Luo,~\IEEEmembership{Fellow,~IEEE}
\thanks{Z. Lin, W. Wei, and X. Chen are with the Department of Electrical and Electronic Engineering, University of Hong Kong, Pok Fu Lam, Hong Kong S.A.R., China (e-mail: linzheng@eee.hku.hk; weiwei@eee.hku.hk; xchen@eee.hku.hk).}
\thanks{Z. Chen and Y. Gao are with the Institute
of Space Internet, Fudan University, Shanghai, China, and the School of Computer Science, Fudan University, Shanghai, China (e-mail: zhechen@fudan.edu.cn; gao.yue@fudan.edu.cn).}
\thanks{Chan-Tong Lam is with the Faculty of Applied Sciences, Macao Polytechnic University, Macau S.A.R., China (e-mail: ctlam@mpu.edu.mo).}
\thanks{Jun Luo is with the School of Computer Engineering, Nanyang Technological University, Singapore (e-mail: junluo@ntu.edu.sg).}
\thanks{\textit{(Corresponding author: Xianhao Chen)}}
}

\markboth{}%
{Shell \MakeLowercase{\textit{et al.}}: Bare Advanced Demo of IEEEtran.cls for IEEE Computer Society Journals}
%



\IEEEtitleabstractindextext{%
\begin{abstract}
As AI models expand in size, it has become increasingly challenging to deploy federated learning (FL) on resource-constrained edge devices. To tackle this issue, \textit{split federated learning} (SFL) has emerged as an FL framework with reduced workload on edge devices via model splitting; it has received extensive attention from the research community in recent years. Nevertheless, most prior works on SFL focus only on a two-tier architecture without harnessing multi-tier cloud-edge computing resources. In this paper, we intend to analyze and optimize the learning performance of SFL under multi-tier systems. Specifically, we propose the hierarchical SFL (HSFL) framework and derive its convergence bound. Based on the theoretical results, we formulate a joint optimization problem for model splitting (MS) and model aggregation (MA). To solve this rather hard problem, we then decompose it into MS and MA sub-problems that can be solved via an iterative descending algorithm. Simulation results demonstrate that the tailored algorithm can effectively optimize MS and MA for SFL within virtually any multi-tier system.
\end{abstract}


}

\maketitle

\IEEEdisplaynontitleabstractindextext

%
\IEEEpeerreviewmaketitle


\section{Introduction}\label{Intro}

The growing prevalence of mobile smart devices and the rapid advancement in information and communications technology (ICT) result in the phenomenal growth of the data generated at the network edge. International Data Corporation (IDC) forecasts that 159.2 zettabytes (ZB) of data will be generated globally in 2024, and this figure is expected to double to 384.6 ZB by 2028~\cite{IDC}. 
To unlock the potential of big data, on-device learning begins to play a crucial role.
In this respect, federated learning (FL)~\cite{mcmahan2017communication} has emerged as a predominant privacy-enhancing on-device learning framework, where participating devices train their local models in parallel on their private datasets and then send the updated models to an FL server for model synchronization~\cite{han2023federated,lin2023fedsn,hu2024accelerating,su2024expediting,yan2023heroes,zhang2024satfed,fang2024automated,zhang2024fedac}.

By extracting intelligence from previously inaccessible private data while remaining the data locally, FL has succeeded in numerous commercial products and applications influential on our daily lives~\cite{hard2018federated,maugeri2024user}. However, while FL offers advantages such as enhanced privacy, its deployment becomes increasingly challenging as machine learning (ML) models continue to grow in size~\cite{lin2024splitlora,qiu2024ifvit,lyu2023optimal}. For instance, the Gemini Nano-2 model~\cite{team2023gemini}, a recently popular on-device large language model, consists of 3.25 billion parameters (equivalent to 3GB in 32-bit floats). The intensive on-device compute costs urgently calls for an alternative/complementary approach to FL. Inherited from split learning (SL)~\cite{vepakomma2018split}, split federated learning (SFL)~\cite{thapa2022splitfed} mitigates the aforementioned issue via split training between a server and edge devices in parallel. Specifically, SFL reduces the client-side computing burden by offloading the primary computing workload to an edge server via model partitioning while periodically aggregating client-side and server-side sub-models following the principle of FL~\cite{kim2024edge,tirana2024workflow,xia2022hsfl,lin2023pushing}. By leveraging device-server collaboration, SFL has garnered significant attention from academia and industry. For instance, Huawei has advocated NET4AI~\cite{huawei2019}, a 6G intelligent system architecture built on SFL, to support future edge ML tasks.



\begin{figure}[t!]
\setlength\abovecaptionskip{6pt}
\setlength\subfigcapskip{0pt}
\centering
\includegraphics[width=0.47\textwidth]{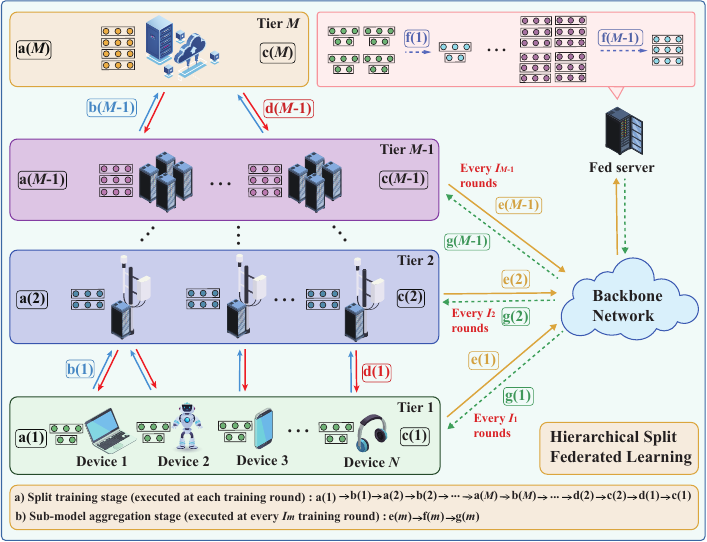}
\caption{The illustration of HSFL over multi-tier computing systems, where 
a($m$) and c($m$) denote $m$-th tier sub-model FP and BP, b($m$) and d($m$) are activations and activations' gradients transmissions between $m$-th tier and ($m$+1)-th tier, e($m$), f($m$) and g($m$) represent $m$-th tier sub-model uplink uploading, aggregation, and downlink transmissions, respectively.}
\label{HSFL}
\end{figure}

\begin{figure}[t!]
    \setlength\abovecaptionskip{6pt}
     \setlength\subfigcapskip{0pt}
     \centering
\subfigure[Two- and three-tier SFL.]{
    \includegraphics[width=.446\columnwidth]{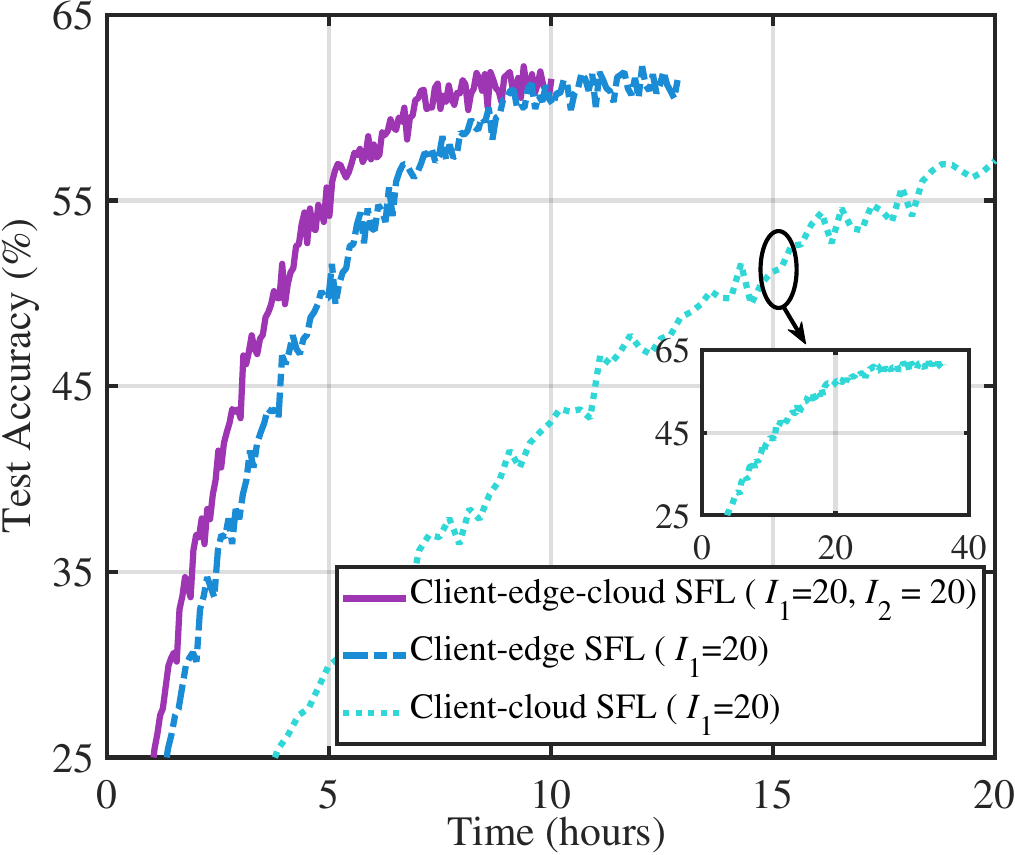}
    \label{sfig:three_two_tier_compare}
}
\subfigure[Test accuracy vs. communications.]{
    \includegraphics[width=.440\columnwidth]{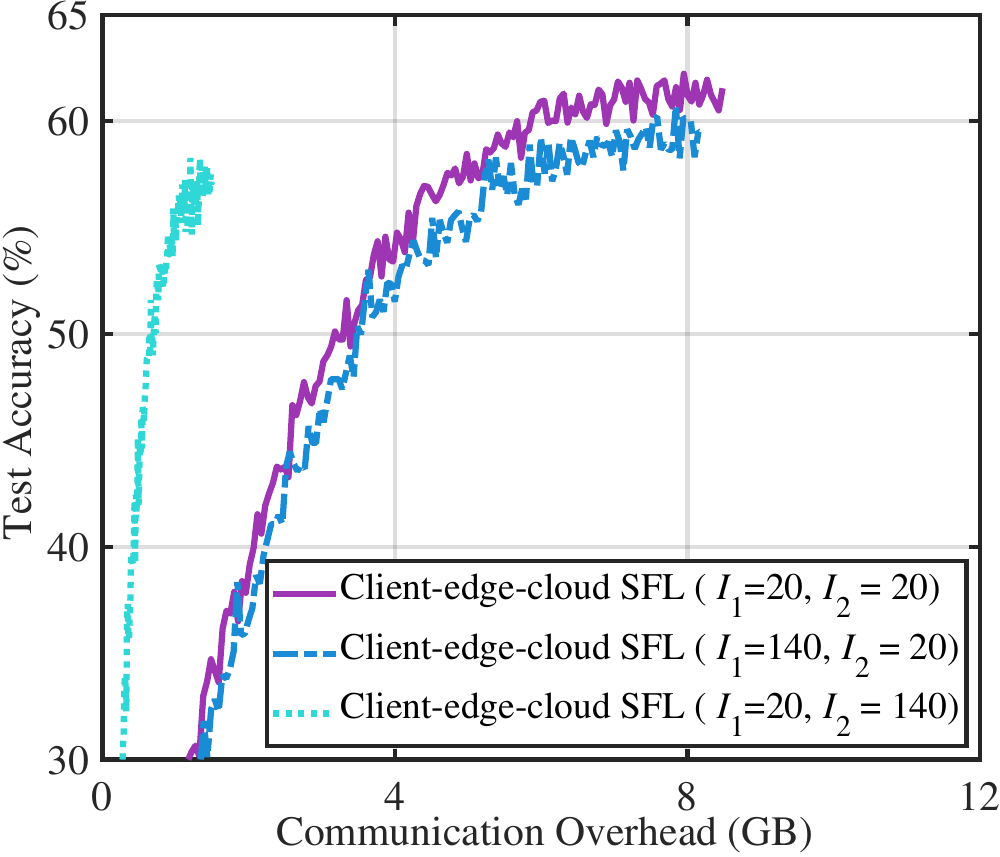}
    \label{sfig:moti_1_accuracy_communication_overhead}
}\\
\subfigure[Training latency vs. model split point.]{
    \includegraphics[width=.455\columnwidth]{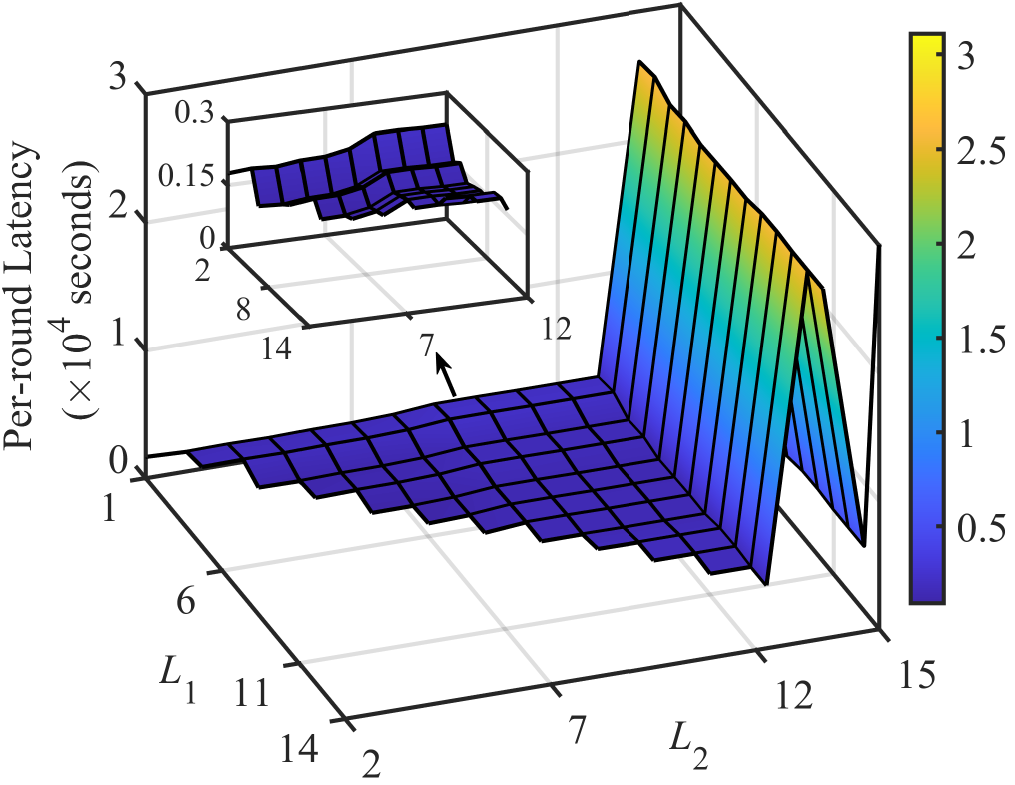}
    \label{sfig:model_split_compu_commu}
}
\subfigure[Test accuracy vs. epochs.]{
    \includegraphics[width=.435\columnwidth]{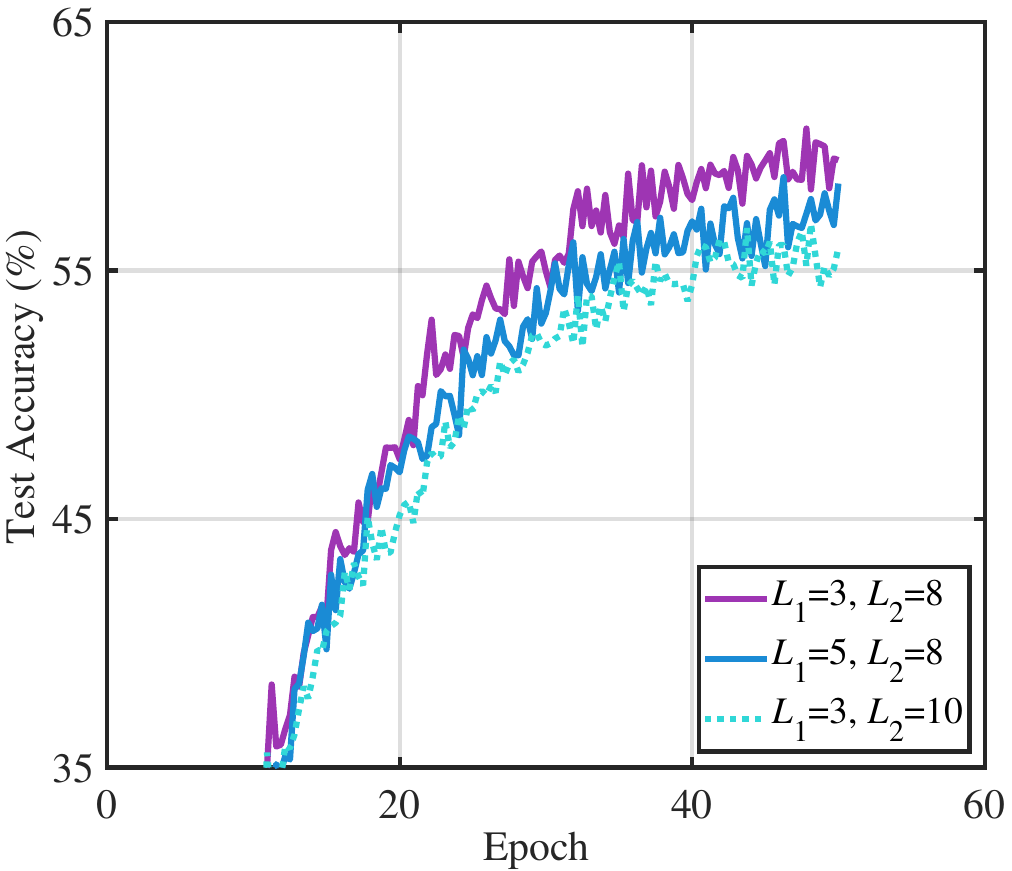}
    \label{sfig:accuracy_epoch}
}
    \caption{The comparison of two-tier and three-tier client-edge-cloud SFL and the impact of sub-model MA and MS on training performance and overhead. Fig.~\ref{sfig:three_two_tier_compare} compares the performance of two- and three-tier SFL for test accuracy versus training time. Fig.~\ref{sfig:moti_1_accuracy_communication_overhead} demonstrates the performance for test accuracy versus communication overhead with the MA intervals for the client-side sub-model ($I_1$) and edge-side sub-model $I_2$, given cutting layers $L_1=3$ and $L_2=8$.  Fig.~\ref{sfig:model_split_compu_commu} presents the per-round end-to-end latency versus cut layers, revealing the complex and significant impact of the cutting layer on communication-computing latency.  Fig.~\ref{sfig:accuracy_epoch} illustrates the performance for test accuracy versus epochs under different cut layers, given MA intervals $I_1=140$ and $I_2=20$, showing that MS has a non-trivial impact on model convergence. The experiment is conducted on the CIFAR-10 dataset under the non-IID setting. The transmission rates between edge devices and a cloud server is set to 15 Mbps~\cite{li2014towards}. The other experimental parameters are consistent with Sec.~\ref{simu_results}.}
    \label{fig:motivation_1}
\end{figure}

Despite the surging research interest in SFL, prior works mostly focus on two-tier SFL systems, consisting of only one split training server. As the number of served edge devices increases, two-tier SFL , however, can hardly support model training for users at a large scale~\cite{lin2023efficient}. To address this scalability issue, a natural idea is to extend SFL to its multi-tier counterpart, called  \underline{h}ierarchical \underline{s}plit \underline{f}ederated \underline{l}earning (HSFL) as illustrated in Fig.~\ref{HSFL}. Specifically, HSFL partitions heavy computing workloads across computing entities across different computing tiers, such as servers on the cloud, at the cell aggregation site, and co-located with macro/small base stations, to accelerate the training process. Taking three-tier client-edge-cloud systems as an example, Fig.~\ref{sfig:three_two_tier_compare} shows that it exhibits faster convergence speed than that of client-edge SFL and client-cloud SFL, by factors of 1.6 and 4.6, respectively, demonstrating the benefits of implementing HSFL.

Given the limited communication-computing resources in multi-tier cloud-edge systems, model splitting (MS) and model aggregation (MA)  significantly impact the performance of HSFL. Several observations on MS and MA can be drawn from our preliminary results in Fig. 2. 

\begin{itemize}
\item {\bf{Impact of MA on HSFL}}: Fig.~\ref{sfig:moti_1_accuracy_communication_overhead} reveals the tradeoff between convergence accuracy and communication overhead by varying the aggregation interval $I_m$ (i.e., $m$-th tier sub-models are aggregated at the FL server every $I_m$ training rounds). Smaller $I_m$ leads to higher convergence accuracy but incurs more communication overhead. Furthermore, it can be observed that the aggregation at different tiers, i.e., $I_1$ and $I_2$, has varied impacts on model convergence, making MA optimization non-trivial.
\item {\bf{Impact of MS on HSFL}}: Fig.~\ref{sfig:accuracy_epoch} illustrates that a shallower cut layer achieves higher training accuracy. This is because a larger portion of the model is synchronized more frequently, with the sub-model on the cloud server always synchronized. In general, a higher tier intends to involve less communication costs for model synchronization, allowing more frequent synchronization or smaller $I_m$. However, making the cut layer shallower may also cause higher communication overhead for exchanging smashed data (i.e., activations and activations' gradients), as shallower layers in some neural networks, such as convolutional neural networks (CNNs), typically have larger output dimensions. 
\end{itemize}



Therefore, optimizing the HSFL system requires quantifying the intertwined effects of MS and MA on training performance.



To fill the void, in this paper, we aim to minimize the training latency of HSFL to achieve the required learning performance in multi-tier resource-constrained systems. To this end, we first derive the convergence bound of HSFL to quantify the impact of MS and MA on training convergence. Different from FL, HSFL features \textit{different synchronization frequencies of sub-models} across various tiers, thus requiring careful theoretical trade-off analysis and optimization. For instance, the top sub-model on the cloud server should always be synchronized as there is only one server on the top tier, whereas the sub-models at the bottom tier involving massive users should only be synchronized infrequently for latency/resource savings. It is also noted that the convergence analysis of HSFL can be considerably more challenging than two-tier SFL as derived in our previous work~\cite{lin2024adaptsfl}, due to the varied aggregation intervals and splitting points across diverse tiers. Subsequently, we formulate a joint optimization problem for MA and MS based on the derived convergence upper bound and develop the corresponding solution approach. The key contributions of this paper are summarized as follows:



\begin{itemize}
\item Given any number of network tiers, we theoretically derive the convergence bound of HSFL to quantify the impact of MS and MA on training performance under multi-tier systems. The theoretical analysis provides a foundation for system optimization of HSFL.
\item Based on the derived upper bound, we formulate a joint optimization problem for MA and MS, aiming to minimize training latency for achieving the target learning performance.
\item We decompose the joint optimization problem into two tractable sub-problems for MS and MA and develop efficient algorithms to solve them optimally, respectively. Then, we devise a block coordinate descent (BCD) based method to obtain efficient sub-optimal solutions to the joint optimization problem.
\item We conduct extensive simulations across various datasets to validate the theoretical analysis and effectiveness of the proposed solutions. 
\end{itemize}

The remainder of this paper is organized as follows. Sec.~\ref{Rel_Work} introduces related work. Sec.~\ref{sec_sfl} elaborates on the HSFL framework. Section~\ref{convergence_HSFL} provides the convergence analysis. We formulate the joint optimization problem in Sec.~\ref{prob_formu} and develop the corresponding solution approach in Sec.~\ref{solu_appro}. Sec.~\ref{sec:implementa} introduces the system implementation, followed by performance evaluation in Sec.~\ref{simu_results}. Finally, concluding remarks are presented in Sec.~\ref{conclu}.

\section{Related Work}\label{Rel_Work}

MA has been extensively studied to enhance the learning performance of federated learning~\cite{luo2024communication,wang2019adaptive,wan2021convergence,shi2022talk,yu2019computation}. Luo~\textit{et al.}~\cite{luo2024communication} propose a communication-efficient FL training framework that enhances training performance while ensuring model convergence through adaptive MA control. Wang~\textit{et al.}~\cite{wang2019adaptive} develop an optimization algorithm that strikes an optimal balance between local update and global parameter aggregation to minimize the training loss under a resource budget. Wan~\textit{et al.}~\cite{wan2021convergence} devise an optimal network scheduling design of local training epochs, data heterogeneity, and partial client participation to enhance the efficiency of FL. Shi~\textit{et al.}~\cite{shi2022talk} and Yu~\textit{et al.}~\cite{yu2019computation} propose an adaptive batch size scheme to minimize the total latency of FL in resource-constrained mobile computing system. The core of the above design lies in optimization based on a derived convergence upper bound. However, these schemes do not apply to HSFL since the model is split across diverse computing tiers and can be implemented with varied MA intervals.

On the other hand, MS plays a vital role in SL, which heavily affects the computing workload, communication overhead, and convergence speed. Some research works have investigated the optimal MS strategy in SL~\cite{wu2023split,kim2023bargaining,lin2023efficient}. Wu~\textit{et al.}~\cite{wu2023split} develop a cluster-based parallel SL framework and design a joint MS, device clustering, and subchannel allocation optimization scheme to expedite the model training. Kim~\textit{et al.}~\cite{kim2023bargaining} devise a personalized SL framework that optimizes MS to balance energy consumption, training time, and data privacy, thereby enhancing the efficiency of model training. Lin~\textit{et al.}~\cite{lin2023efficient} propose an efficient parallel SL framework featuring last-layer gradient aggregation and devise a joint MS and resource management strategy to minimize the per-round training latency. Unfortunately, these studies focus on communication-computing latency optimization without taking into account the impact of MS on training convergence.

There are some recent works on the optimization of HSFL~\cite{khan2023joint,xia2022hsfl}. However, these works have not established the convergence analysis of HSFL and studied the optimization of MA and MS to accelerate model convergence. To our knowledge, our recent work~\cite{lin2024adaptsfl} is the only work investigating the joint optimization of MA and MS to improve the training performance of SFL based on theoretical convergence analysis. However, neither the convergence analysis nor the system optimization of SFL can be directly applied to HSFL due to the considerably more intricate impact of MA intervals and MS on the training performance in multi-tier systems.


\section{HSFL Framework}\label{sec_sfl}
In this section, we begin by introducing the system model and then present the proposed HSFL framework.

\subsection{System Model}\label{sys_model}
As illustrated in Fig.~\ref{HSFL}, we consider a typical scenario of HSFL in multi-tier computing systems. Without loss of generality, we consider a hierarchical computing system with $M$ tiers, where different numbers of computing entities are located in each tier. The fed server~\cite{thapa2022splitfed} is responsible for sub-model synchronization, periodically aggregating sub-models from the same tier. The set of participating edge devices is denoted by $\mathcal{N} = \left\{ {1,2,...,N} \right\}$, where $N$ is the number of edge devices. Each edge device $n$ has its own local dataset ${\mathcal{D}_n} = \left\{ {{{\bf{x}}_{n,k}},{y_{n,k}}} \right\}_{k = 1}^{\left| {{{\cal D}_n}} \right|}$, where ${{\bf{x}}_{n,k}}$ and ${{{y}}_{n,k}} $ are the $k$-th input data and its corresponding label. The set of computing entities in the $m$-th tier is denoted by ${\mathcal{J}}_m = \left\{ {1,2,...,J_m} \right\}$, where $J_m$ represent the number of computing entities. In most cases, the top tier $M$ only has one cloud server, i.e., $J_M=1$\footnote{although the ``cloud'' possesses a massive number of servers/GPUs, the inter-cloud communication latency can be ignored compared with wireless/backbone networks, which hence can be regarded as one logical server.}. The set of tiers is denoted by $\mathcal{M}={\{{1, 2, ..., M}\}}$. The sub-model of client $n$ at the $m$-th tier is denoted by ${{\bf{w}}_{m,n}}$. Therefore, the global model is represented as ${\bf{w}} = \left[ {{\bf{w}}_{1,n}}; {{\bf{w}}_{2,n}; ...; {{\bf{w}}_{{M},n}}} \right] $. The total number of layers for the global model is denoted by $L = \sum\limits_{m = 1}^{{M}} {{L_m}}$, where $L_m$ is the number of layers for $m$-th tier sub-models.
The goal of SL is to find the optimal model parameters ${{\bf{w}}^{\bf{*}}}$ by minimizing the following finite-sum non-convex global loss function across all participating devices:
{\begin{equation}\label{minimiaze_loss_function}
\begin{aligned}
\mathop {\min }\limits_{\bf{w}} f\left( {\bf{w}} \right) \buildrel \Delta \over = \mathop {\min }\limits_{\bf{w}} {\frac{{{1}}}{{N}}} \sum\limits_{n = 1}^{{N}} {f_n}({\bf{w}}),
\end{aligned}
\end{equation}}
where ${f_n}\left( {\bf{w}} \right)$ is the local loss function of edge device $n$, i.e., ${f_n}\left( {\bf{w}} \right) \buildrel \Delta \over = {\mathbb{E}_{{\xi _n} \sim {\mathcal{D}_n}}}[{F_n}\left( {{\bf{w}};{\xi _n}} \right)]$, ${F_n}\left( {{\bf{w}};{\xi _n}} \right)$ denotes the stochastic loss function of edge device $n$ for a mini-batch data samples, and $\xi _n$ is the training randomness of local dataset $\mathcal{D}_n$, stemming from stochasticity of the data ordering caused by mini-batch data sampling from $\mathcal{D}_n$. Following the standard stochastic optimization setting~\cite{Karimired2018,yu2019linear,karimireddy2020scaffold}, we assume the existence of unbiased gradient estimators, i.e., $\mathbb{E}_{\xi_{n}^{t}\sim \mathcal{D}_{n}}[\nabla F_{n}(\mathbf{w}^{t-1}_{n};\xi_{n}^{t}) \vert \boldsymbol{\xi}^{[t-1]}] = \nabla f_{n}(\mathbf{w}^{t-1}_{n})$ for any fixed $\mathbf{w}^{t-1}_{n}$, where $\boldsymbol{\xi}^{[t-1]} \overset{\Delta}{=} [\xi_{n}^{\tau}]_{n\in\{1,2,\ldots,N\}, \tau\in\{1,\ldots,t-1\}}$ represents all randomness up to training round $t-1$.

To solve problem~\eqref{minimiaze_loss_function}, the two-tier conventional SFL framework adopts the co-training paradigm of edge devices and an edge server. However, as ML models and the number of clients scale up, the edge server struggles to handle the heavy computing workload from multiple edge devices, leading to intolerable training latency and costs. Motivated by this, we propose the HSFL framework and design the MS and MA scheme for accelerating the training process.

\subsection{The HSFL Framework}\label{subsec_hsfl}
This section elaborates on the workflow of the HSFL framework. The heart of HSFL lies in the optimal control of MS and MA, which enables efficient SFL co-training of diverse computing entities across different tiers.

Before model training begins, the fed server initializes the ML model, partitions it into $M$ sub-models via MS (described in Sec.~\ref{solu_appro}), and determines the optimal sub-model MA intervals for each tier (discussed in Sec.~\ref{solu_appro}). As shown in Fig.~\ref{HSFL}, the HSFL training process consists of two primary stages: split training and sub-model aggregation. 
The split training is executed each training round, while the sub-model aggregation of $m$-th tier occurs every $I_m$ training round. 
For any training round $t \in \mathcal{R} = \left\{ {1,2,...,R} \right\}$, these two stages are detailed as follows.


\textit{A. The split training stage:} The split training stage involves the sub-model update for each tier. During this stage, each edge device/server conducts forward propagation (FP) on its sub-models and sends the generated activations to a higher-tier device/server for further sub-model FP. After sub-model FP reaches the top tier (i.e., the $M$-th tier), each computing entity in the $M$-th tier executes sub-model backward pass (BP) and then transmits the activations' gradients to the lower-tier device/server for performing sub-model BP. This process continues tier by tier until edge devices, at the lowest tier, complete their sub-model BP. 
To better understand this stage, we separate it into five steps and discuss each in detail below.


\textit{a) Sub-model forward propagation:} This step involves FP of all participating computing entities across diverse tiers. For the lowest tier, each edge device $n$ randomly draws a mini-batch ${\mathcal{B}_n} \subseteq {\mathcal{D}_n}$ containing $b$ data samples from its local dataset ${\mathcal{D}_n}$ to execute FP. For higher tiers, the $m$-th tier computing entities utilize activations from connected entities in ($m$-$1$)-th tier, instead of the raw data, for FP. After the sub-model FP is completed, activations are generated at the cut layer. The activations generated by sub-models of client $n$ at the $m$-th tier are represented as
\begin{align}\label{stage_1_1}
{\bf{a}}_{m,n}^t = \left\{ {\begin{array}{*{20}{c}}
{\varphi \left( {{\bf{x}}_n^t;{\bf{w}}_{1,n}^{t - 1}} \right),}&{m = 1}\\
{\varphi \left( {{\bf{a}}_{m-1,n}^t;{\bf{w}}_{m,n}^{t - 1}} \right),}&{\!\!\! 2 \le m \le M}
\end{array}} \right.\!\!, \forall n \in {\mathcal{N}},
\end{align}
where ${\bf{x}}_n^t$ denotes a mini-batch input data of client $n$ in $t$-th training round, and $\varphi\left( {{{x}};{{w}}} \right)$ maps the relationship between input ${{x}}$ and its predicted value given model parameter ${{w}}$.

\textit{b) Activation uploading:} 
After completion of the sub-model FP, each computing entity sends its activations and corresponding labels to the connected entities in the tier above (typically over a wireless channel). These collected activations are then utilized to empower the sub-model FP of higher-tier computing entities.



\textit{c) Sub-model backward pass:} After sub-model FP reaches the top tier (i.e., the $M$-th tier), each computing entity in the $M$-th tier executes sub-model BP based on the loss function value. For lower tiers, the $m$-th tier computing entities leverage activations' gradients from connected entities in ($m$+$1$)-th tier for sub-model BP. Since a higher-tier computing entity may simultaneously handle sub-models of multiple clients, these sub-models can always be synchronized every training round without incurring any communication overhead. 
Consequently, the model parameters of the $j$-th computing entity in the $m$-th tier can be updated through\footnote{The computing entities can execute sub-model updates of multiple clients in either a serial or parallel fashion, which does not affect training performance. Here, we consider the parallel fashion.}
\begin{align}\label{stage_5_2}
{\bf{w}}_m^{j,t} = \frac{1}{{{N^j_m}}}\sum\limits_{n \in {{\mathcal{N}}^j_{m}}} {{\bf{w}}_{m,n}^t},\forall j \in {\mathcal{J}}_m, \forall m \in {\mathcal{M}}\setminus{\{{M}\}},
\end{align}
where ${{\mathcal{N}}^j_{m}}$ is the set of clients corresponding to the sub-models in the $j$-th computing entity of the $m$-th layer, ${{N}^j_{m}}$ denotes the number of sub-models, ${\bf{w}}_{m,n}^t \leftarrow {\bf{w}}_{m,n}^{t - 1} - \gamma {\nabla _{{{\bf{w}}_m}}}{F_n}({\bf{w}}_{m,n}^{t - 1};\xi _n^t)$ is the $m$-th tier sub-model of client $n$, {${\nabla _{\bf{w}}}F({\bf{w}}; \xi)$ represents the gradient of loss function $F({\bf{w}}; \xi)$ with respect to the model parameter ${\bf{w}}$, and $\gamma$ is the learning rate. After sub-model BP is completed, activations' gradients are generated at the cut
layer.

}

\textit{d) Downloading of activations' gradients:}  
After completion of the sub-model BP, each computing entity sends its activations' gradients and corresponding labels to the connected entities in the tier below. The lower-tier entities then use the received activations' gradients to conduct the sub-model BP.


\textit{B. The sub-model aggregation stage:} The sub-model aggregation stage focuses on aggregating sub-models from each tier on the fed server. The sub-model MA intervals vary across diverse tiers (i.e., sub-models from  $m$-th tier are aggregated every $I_m$ training rounds). For every $I_m$ training round, computing entities in the $m$-th tier send their sub-models to the fed server. In practice, the fed server for multi-tier HSFL systems is an application server responsible for model aggregation for a relatively large geographic region, which, hence, is often a cloud server over the Internet. 
Note that $I_m$ is not fixed throughout the training process and can be adjusted after each sub-model MA based on current wireless/wired network conditions and training states (shown in Sec.~\ref{solu_appro}). The sub-model aggregation stage consists of the following three steps.

\textit{e) Sub-model uploading:} In this step, computing entities in the same tier simultaneously send their respective sub-models to the fed server over the wireless/wired links.

\begin{algorithm}[t!]
	\renewcommand{\algorithmicrequire}{\textbf{Input:}}
	\renewcommand{\algorithmicensure}{\textbf{Output:}}
	\caption{HSFL Training Procedure}\label{HSFL_procedure}
	\begin{algorithmic}[1]
 \REQUIRE   $b$, $\gamma$, $E$, ${\cal N}_m$, and ${\cal{D}}_n$, ${\cal M} = \{ 1,2,...,M - 1\}$.
		\ENSURE ${{\bf{w}}^{\bf{*}}}$. 
		\STATE Initialization: ${{\bf{w}}^{{0}}}$, ${{\bf{w}}_n^{{0}}} \leftarrow {{\bf{w}}^{{0}}}$, $I^0 \leftarrow 1$, $\tau \leftarrow 0$, $\rho  \leftarrow 0$, and ${\cal M}' \leftarrow {\cal M}$.
          \FOR {$t=1, 2,..., R$} 
          \FOR {$m=1, 2,..., M-1$}
          \STATE
          \STATE \textbf{/** {Runs} {on} {computing} {entities} **/}
          \FORALL {sub-model ${{\bf{w}}^{t-1}_{m,n}}$ in parallel}
            \STATE  ${\bf{a}}_{m,n}^t = \left\{ {\begin{array}{*{20}{c}}
            {\varphi \left( {{\bf{x}}_n^t;{\bf{w}}_{1,n}^{t - 1}} \right),}&{m = 1}\\
            {\varphi \left( {{\bf{a}}_{m-1,n}^t;{\bf{w}}_{m,n}^{t - 1}} \right),}&{2 \le m \le M}
            \end{array}} \right.$
            \STATE Send $\left( {{{\bf{a}}^t_{m,n}},{\bf{y}}_{m,n}^t} \right)$ to the ($m$+$1$)-th tier connected entities
          \ENDFOR
          \ENDFOR

	   \STATE
          \FOR {$m=M, M-1,..., 1$}
          \FORALL {sub-model ${{\bf{w}}^{t-1}_{m,n}}$ in parallel}
          \STATE ${\bf{w}}_{m,n}^t \leftarrow {\bf{w}}_{m,n}^{t - 1} - \gamma {\nabla _{{{\bf{w}}_m}}}{F_n}({\bf{w}}_{m,n}^{t - 1};\xi _n^t)$ 
          \STATE Send activations' gradients  to the ($m$-$1$)-th tier connected entities
          \ENDFOR
          \FOR {$j=1, 2, ..., J_m$}
          \STATE ${\bf{w}}_m^{j,t} = \frac{1}{{{N^j_m}}}\!\!\sum\limits_{n \in {{\mathcal{N}}^j_{m}}}\!\!\! {{\bf{w}}_{m,n}^t}$
          \ENDFOR
          \ENDFOR
        \STATE
        \STATE \textbf{/** {Runs} {on} the {fed} {server} **/}
        \IF{${\cal M}' = \emptyset$}
        \STATE  Determine ${\bf{I}} = [I_1^{{\tau _1}},I_2^{{\tau _2}},...,I_{M - 1}^{{\tau _{M - 1}}}]$ and ${\boldsymbol{\mu}}$ based on~\textbf{Algorithm~\ref{BCD-based}}
        \STATE $\boldsymbol{\tau}  = \boldsymbol{\tau}  + \bf{1}$ $(\boldsymbol{\tau}=[{\tau _1},{\tau _2},...,{\tau _{M - 1}}])$, ${\cal M}' \leftarrow {\cal M}$
     \ELSIF{$( {t - \!\!\!\sum\limits_{\tau  = 1}^{{\tau _m}-1} {I_m^\tau } } )$ mod $I_m^{\tau_m}$ $=0$}
     \STATE ${\bf{w}}_m^t = \frac{{N_m^j}}{N}\sum\limits_{j = 1}^{{J_m}} {{\bf{w}}_m^{j,t}}$, ${\cal M}' = {\cal M}' \setminus \{m\}$
     \STATE Determine $I_m^{{\tau_m}+1}$ based on~\textbf{Theorem 2}
     \STATE ${\tau_m} \leftarrow {\tau_m}+1$
        \ENDIF
        \STATE $t  \leftarrow t+1$
        \ENDFOR
                
	\end{algorithmic}  
\end{algorithm}

\textit{f) Sub-model aggregation:} The fed server aggregates the sub-models of different computing entities from the same layer. For the $m$-th tier, the aggregated sub-model is denoted by
\begin{align}\label{h_c_define}
{\bf{\overline w}}_m^t = \sum\limits_{j = 1}^{{J_m}} \frac{{N_m^j}}{N} {{\bf{w}}_m^{j,t}}, \;\;\; \forall m \in {\mathcal{M}}\setminus{\{{M}\}}.
\end{align}

\textit{g) Sub-model downloading:} After completing the sub-model aggregation, the fed server sends the updated sub-model to the corresponding edge devices/server.

The HSFL training procedure is outlined in \textbf{Algorithm~\ref{HSFL_procedure}}.

\section{Convergence Analysis of the HSFL Framework}\label{convergence_HSFL}

In this section, we conduct the convergence analysis of HSFL to quantify the impact of MS and MA on training convergence, which serves as the theoretical foundation for developing an efficient iterative optimization method in Sec.~\ref{solu_appro}.

In line with seminal works in distributed stochastic optimization~\cite{zhang2012communication,lian2017can,mania2017perturbed,lin2018don}, we focus on the convergence analysis of the aggregated version of individual solutions ${\bf{\overline w}} = [{\bf{\overline w}}_1; {\bf{\overline w}}_2; ...; {\bf{\overline w}}_M ]$, where ${\bf{\overline w}}_m$ denotes the aggregated sub-model of $m$-th tier. 
For notational simplicity, we denote the gradient of the $m$-th tier of client $n$ as ${\bf{g}}_{m,n}^t = {{\nabla _{{{\bf{w}}_m}}}{F_n}({\bf{w}}_{m,n}^{t - 1};\xi _n^t)}$.
To analyze the convergence rate of HSFL, we consider the following two standard assumptions on the loss functions:
\begin{assumption}
[\textit{Smoothness}] \label{asp:1} \textit{Each local loss function ${f_n}\left( {\bf{w}} \right)$ is differentiable and $\beta $-smooth., i.e., for any $\mathbf{w}$ and ${{\bf{w'}}}$, we have}
{ \begin{equation}
\begin{aligned}
\left\| {\nabla_{\bf{ w}} {f_n}\left( {\bf{w}} \right) - \nabla_{\bf{w}} {f_n}\left( {\bf{w'}} \right)} \right\| \le \beta \left\| {{\bf{w}} - {\bf{w'}}} \right\|,\; n \in \mathcal{N}.
\end{aligned}
\end{equation}}
\end{assumption}
\begin{assumption}
[\textit{Bounded variances and second moments}] \label{asp:2} \textit{The variance and second moments of stochastic gradients for each layer have upper bounds, i.e., }
{ \begin{equation}
\begin{aligned}
\mathbb{E}_{\xi_{n}\sim \mathcal{D}_{n}} \Vert \nabla_{{\bf{w}}_l} F_{n}({\mathbf{w}}_l; \xi_{n}) \!-\! \nabla_{{\bf{w}}_l} f_{n}({\mathbf{w}}_l)\Vert^{2} \!\leq \! {\sigma _l^2}, \forall {\mathbf{w}}_l, n \in \mathcal{N},
\end{aligned}
\end{equation}}
{ \begin{equation}
\begin{aligned}
\mathbb{E}_{\xi_{n}\sim \mathcal{D}_{n}} \Vert \nabla_{{\mathbf{w}}_l} F_{n}({\mathbf{w}}_l; \xi_{n}) \Vert^{2} \leq {G _l^2}, \; \forall {\mathbf{w}}_l, n \in \mathcal{N},
\end{aligned}
\end{equation}}
\textit{where ${\bf w}_l$ is the $l$-th layer of model $\bf w$ and ${\sigma _l^2}$ and ${G_l^2}$ denote the bounded variance and second order moments of ${\bf w}_l$.}  \end{assumption}

\begin{lemma} \label{lm:diff-avg-per-node}
Under {\bf{Assumption \ref{asp:1}}, \bf{Algorithm \ref{HSFL_procedure}}} ensures
\begin{align*}
\mathbb{E} [\Vert {\mathbf{\overline w}}_m^{t} - \mathbf{w}^{t}_{m,n}\Vert^{2}] \leq {\mathbbm{1}}_{\{I_m > 1\}} 4\gamma^{2} I_m^{2} \sum\limits_{l = L_{m-1}+1}^{L_{m-1}+L_m}  {G _l^2}, \forall n, \forall t,
\end{align*}
\textit{where ${\mathbbm{1}}_{\{\cdot\}}$ denotes the indicator function, $I_m$ represents model aggregation interval of $m$-th tier and ${\mathbf{\overline w}}_m^{t}$ is defined in {Eqn. \eqref{h_c_define}}.}
\end{lemma}

\begin{proof}
See Appendix A.
\end{proof}

Under the above assumptions, the following theorem holds for the training process of HSFL:

\begin{theorem}
\textit{Supposing that learning rate $\gamma$ satisfies $0 < \gamma \leq \frac{1}{\beta}$, then for all $R\geq 1$, we have }
{ \begin{equation}\label{convergence_bound}
\small \begin{aligned}
&\frac{1}{R}\sum\limits_{t = 1}^R \mathbb{E}  [{\Vert\nabla _{\bf{ w}}}f({{\bf{\overline w}}^{t - 1}})\Vert{^2}]\! \\& \le \frac{2\vartheta}{{\gamma R}}\! +\! \frac{{\beta \gamma \sum\limits_{l = 1}^L {\sigma _l^2} }}{N}
  + 4{\beta ^2}{\gamma ^2}\!\sum\limits_{m = 1}^{M - 1} \!\!{\bigg( {{\mathbbm{1}}_{\{I_m > 1\}} I_m^2 \!\!\sum\limits_{l = {L_{m - 1}} + 1}^{{L_{m - 1}} + {L_m}} \!\! G_l^2} \bigg)},
\end{aligned}
\end{equation}}%
where $\vartheta  = f({{\bf{\overline w}}^0}) - {f^ * }$, $L$ and $f^{\ast}$ represent the total number of global model layers and the optimal objective value of problem~\eqref{minimiaze_loss_function}. 
\end{theorem}

\begin{proof}
See Appendix B.
\end{proof}
Substituting Eqn.~\eqref{convergence_bound} into Eqn.~\eqref{accuracy_cons_corollary} yields {\bf{Corollary 1}}, yielding a lower bound on the number of training rounds to achieve target convergence accuracy.
\begin{corollary}\label{theorem1}
The number of {training rounds $R$} for achieving target convergence accuracy $\varepsilon$, i.e., satisfying
{ \small \begin{equation}\label{accuracy_cons_corollary}
\frac{1}{R} \sum_{t=1}^{R} \mathbb{E} [\Vert \nabla_{\bf{w}} f({\mathbf{\overline w}}^{t-1})\Vert^{2}] \le \varepsilon,
\end{equation}}
is given by
{ \small \begin{equation}\label{lowest_com_num}
R  \ge \frac{2 \vartheta}{{\gamma \bigg( {\varepsilon  - \frac{{\beta \gamma \sum\limits_{l = 1}^L {\sigma _l^2}}}{N} - 4{\beta ^2}{\gamma ^2}\!\sum\limits_{m = 1}^{M - 1}\!\! {\bigg( {{\mathbbm{1}}_{\{I_m > 1\}} I_m^2\!\! \sum\limits_{l = {L_{m - 1}} + 1}^{{L_{m - 1}} + {L_m}}  \!\!G_l^2}} \bigg)}\bigg)}}.
\end{equation}}
\end{corollary}
{\textbf{Insight 1}:} Eqn.~\eqref{lowest_com_num} shows that the number of training rounds required to achieve a target convergence accuracy $\varepsilon$ decreases with shortening sub-model aggregation interval $I_m$. This indicates that more frequent sub-model aggregation leads to faster model convergence. Similarly, for a fixed number of training rounds $R$, decreasing sub-model aggregation interval results in higher convergence accuracy (i.e., smaller $\varepsilon$). 

{\textbf{Insight 2}}: The cut layer selection $L_m$ affects the impact of aggregation interval $I_m$ on the model convergence, as it determines which layers to be aggregated. Moreover, considering a tier with only one server at tier $m$, i.e., $J_m=1$, the last term for the $m$-th tier disappears, which implies faster convergence. This is because that a single server makes the neural network layers always co-located and synchronized.

The above observations are consistent with the experimental results shown in Fig.~\ref{fig:motivation_1}. Since sub-model aggregation involves exchanging models between the fed server and computing entities, training performance improvement comes at the cost of increased communication overhead, leading to longer latency per training round. Therefore,  optimizing MS and sub-model MA under communication and computing constraints is essential for expediting SFL.

\section{Problem Formulation}\label{prob_formu}

In this section, we formulate a joint MA and MS optimization problem based on the derived convergence bound in Sec.~\ref{convergence_HSFL}. The objective is to minimize the training latency of HSFL to achieve target learning performance in a resource-constrained multi-tier computing system. Subsequently, we devise an efficient MS and MA strategy in Sec.~\ref{solu_appro}. For clarity, the decision variables and definitions are listed below.

\begin{itemize}
\item $\bf I$: $I_m \in {\mathbb{N}^ + }$ is the sub-model MA decision variable, indicating that sub-models at $m$-th tier are aggregated on the fed server every $I_m$ training rounds. ${\bf I}  = \left[ {{I _{1}},{I_{2}},...,{I_{M-1}}} \right]$ represents the collection of sub-model MA decisions.
\item ${\boldsymbol\mu}$: $\mu_{m,l} \in \left\{ {0,1} \right\}$ denotes the MS decision variable, where $\mu_{m,l}=1$ indicates that the $l$-th neural network layer is selected as the cut layer between $m$-th and ($m$+1)-th tier, and 0 otherwise.
${\boldsymbol\mu}  = \left[ {{\mu _{1,1}},{\mu _{1,2}},...,{\mu _{M-1,L}}} \right]$ represents the collection of MS decisions.
\end{itemize}

\subsection{Training Latency Analysis}
In this section, we analyze the training latency of HSFL. Without loss of generality, we focus on one training round for analysis. In each training round, edge device $n$ randomly draws a mini-batch ${\mathcal{B}_n} \subseteq {\mathcal{D}_n}$ with $b$ data samples from its local dataset for model training. To begin, we provide a detailed latency analysis of the split training stage.

\textit{a) Sub-model forward propagation latency:} The computing entities in each tier utilize the local dataset/received activations from the lower tier to conduct their sub-model FP. A computing entity can execute FP and BP for multiple edge devices in a serial or parallel fashion. Here, we consider the parallel fashion. The computing workload (in FLOPs) of sub-model FP in the $m$-th tier per mini-batch $b$ is denoted by $\Phi _{m}^F\left( {b, \boldsymbol{\mu}} \right) = \sum\limits_{l = 1}^L {\left( {{\mu _{m,l}} - {\mu _{m - 1,l}}} \right)} {\rho _l(b)}$, where ${\rho _l(b)}$ is the FP computing workload of propagating the first $l$ layer of neural network for one mini-batch $b$. Thus, the FP latency for client $n$'s sub-model located in the $m$-th tier is given by 
\begin{align}\label{client_FP_latency}
T_{m,n}^F = \frac{{{\kern 1pt} \Phi _{m}^F\left( {b, \boldsymbol{\mu }} \right)}}{{{f_{m,n}}}},\;\;\;\;\;\forall n \in \mathcal{N}, \forall m \in {\mathcal{M}},
\end{align}
where ${{ f_{m,n}}}$ denotes the computing capability allocated for processing the sub-model of client $n$ located in the $m$-th tier\footnote{We assume the association of clients and computing entities as well as resource allocation in each tier are predetermined. Thus, we can obtain the computing capabilities and data rate allocated for a client's sub-model in each tier and omit the index of the specific computing entity hosting the sub-model.} (namely, the number of float-point operations per second (FLOPS)).

\textit{b) Activation uploading latency:}
After the completion of sub-model FP, computing entities send the activations generated at the cut layer to those at the higher tier to continue model training. The data size (in bits) of activations of $m$-th tier is represented as  ${\Gamma ^A_{m}}\left( {\boldsymbol{\mu }} \right) = \sum\limits_{l = 1}^L {{\mu _{m,l}}} {\psi _l}$, where ${\psi _l}$ denotes the data size of activations at the cut layer $l$. The activation transmission latency for  client $n$'s sub-model
located in the $m$-th tier can be calculated as
\begin{align}\label{smashed_trans_latency}
T_{m,n}^{A} = \frac{{b\Gamma^A_{m} \left( \boldsymbol{\mu}  \right)}}{{r_{m,n}^{A}}},\;\;\;\;\;\forall n \in \mathcal{N}, \forall m \in {\mathcal{M}}\setminus{\{{M}\}},
\end{align}
where ${r_{m,n}^{A}}$ is the uplink transmission rate for transmitting the activations corresponding to client $n$'s sub-model in the $m$-th tier to the corresponding upstream computing entity at ($m+1$)-th tier.

\textit{c) Sub-model backward pass latency:}
After sub-model FP reaches the top tier (i.e., the $M$-th tier), computing entities at diverse tiers execute sub-model BP based on loss function value/received activations' gradients from the higher tier. Let $\Phi _{m}^B\left( {b, \boldsymbol{\mu}} \right) = \sum\limits_{l = 1}^L {\left( {{\mu _{m,l}} - {\mu _{m - 1,l}}} \right)} {\varpi _l(b)}$ denote the computing workload of sub-model BP of $m$-th tier per mini-batch $b$, where ${\varpi _l(b)}$ is the BP computing workload of propagating the first $l$ layer of neural network for one mini-batch $b$. Thus, the BP latency for  client $n$'s sub-model located in the $m$-th tier can be obtained from
\begin{align}\label{server_BP_latency}
T_{m,n}^B = \frac{{{\kern 1pt} \Phi _{m}^B\left( {b, \boldsymbol{\mu }} \right)}}{{{f_{m,n}}}},\;\;\;\;\;\forall n \in \mathcal{N}, \forall m \in {\mathcal{M}}.
\end{align}

\textit{d) Downloading latency of activations' gradients:}
After the sub-model BP is completed, computing entities transmit the activations' gradients to those at a lower tier for further model training. Let $\Gamma^G_{m} \left( \boldsymbol{\mu}  \right) = \sum\limits_{l = 1}^{L} {{\mu _{m,l}}} {\chi _l}$ represent the data size of activations' gradients of the $m$-th tier, where ${\chi _l}$ denotes the data size of activations' gradients at cut layer $l$. Therefore, the transmission latency of activations' gradients for  client $n$'s sub-model  at the ($m$+1) tier is expressed as
\begin{align}\label{downlink_latency}
T_{m,n}^{G} = \frac{{b\Gamma^G_{m} \left( \boldsymbol{\mu}  \right)}}{{r_{m,n}^{G}}},\;\;\;\;\;\forall n \in {\cal N}, \forall m \in {\mathcal{M}}\setminus{\{{M}\}},
\end{align}
where ${r_{m,n}^{G}}$ is the downlink transmission rate for transmitting the activations' gradients corresponding to  client $n$'s sub-model at the ($m$+1) tier to the corresponding downstream computing entity in the $m$-th tier.

Next, we analyze the latency of the sub-model aggregation.

\textit{e) Sub-model uploading latency:} 
The computing entities at the same tier send their aggregated sub-models\footnote{For each computing entity, it aggregates the sub-models of its hosted clients before uploading to the fed server. The aggregation latency is ignored.} to the fed server for MA. Let $\Lambda_{m} \left( \boldsymbol{\mu}  \right) =  \sum\limits_{l = 1}^L {\left( {{\mu _{m,l}} - {\mu _{m - 1,l}}} \right)} {\delta _l}$ denote the data size of sub-model of the $m$-thtier, where ${\delta _l}$ is the data size of the sub-model with the first $l$ layers. Therefore, the sub-model uploading latency of the $j$-th computing entity located in the $m$-th tier is expressed as 
{\begin{align}\label{smashed_trans_latency}
T_{m}^{j,U} \!=\!{\mathbbm{1}}_{\{J_m > 1\}} \frac{{{\Lambda _m}\left( {\boldsymbol{\mu }} \right)}}{{r_{m}^{j,U}}},\;\;\; \forall j \in {\mathcal{J}}_m,  \forall m \in {\mathcal{M}}\setminus{\{{M}\}},
\end{align}}
where $r_{m}^{j,U}$ is the uplink data rate for transferring sub-model from the $j$-th computing entity located in the $m$-th tier to the fed server, and ${\mathbbm{1}}_{\{J_m > 1\}}$ implies sub-model uploading for aggregation is needed only when there is more than one server in the $m$-th tier.

\textit{f) Sub-model model aggregation:} The fed server aggregates the received sub-models from the same tier. For simplicity, the sub-model aggregation latency for this part is ignored, as it is negligible compared to other steps~\cite{shi2020joint,xia2021federated}.

\textit{g) Sub-model downloading latency:} After completing sub-model MA, the fed server sends the updated sub-model to the corresponding computing entities. Similarly, the sub-model downlink transmission latency of the $j$-th computing entity located in the $m$-th tier is calculated by 
\begin{align}\label{smashed_trans_latency}
T_{m}^{j,D} = {\mathbbm{1}}_{\{J_m > 1\}} \frac{{{\Lambda _{m}}\left( {\boldsymbol{\mu }} \right)}}{{r_{m}^{j,D}}},\; \forall j \in {\mathcal{J}}_m, \forall m \in {\mathcal{M}}\setminus{\{{M}\}},
\end{align}
where $r_{m}^{j, D}$ is the downlink rate for transmitting sub-model from the fed server to $j$-th computing entity located in the $m$-th tier.

\subsection{Joint Model Aggregation and Model Splitting Problem Formulation}

\begin{figure}[t]
    \setlength\abovecaptionskip{6pt}
     \setlength\subfigcapskip{0pt}
     \centering
\subfigure[Split training.]{
    \includegraphics[width=6.08cm, height=1.85cm]{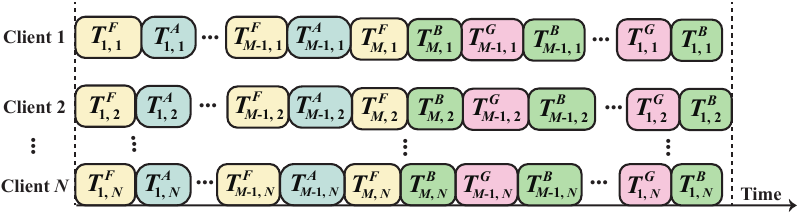}
    \label{sfig:split_training}
}
\subfigure[Sub-model aggregation for tier $m$.]{
    \includegraphics[width=2.14cm, height=1.85cm]{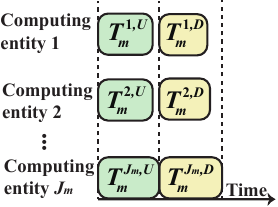}
    \label{sfig:client_side_aggrega}
}
    \caption{An illustration of split training and sub-model aggregation stages.}
    \label{fig:training_2_stage}
\end{figure}

In this section, we formulate the joint MA and MS optimization problem to minimize the training latency of HSFL for model convergence in a resource-constrained multi-tier system. As illustrated in Fig.~\ref{fig:training_2_stage}, the per-round latency of split training is calculated as 
{{ \begin{align}
T_S\!\left( {\boldsymbol{\mu} } \right) \!=\!\mathop {\max }\limits_n \bigg\{ {\! \sum\limits_{m = 1}^M {T_{m,n}^F \!+\! \!\sum\limits_{m = 1}^{M - 1} {T_{m,n}^A \!+ \!\!\sum\limits_{m = 1}^M {T_{m,n}^B \!+\! \!\!\sum\limits_{m = 1}^{M - 1} {T_{m,n}^G} \!} } } } \bigg\}, 
\end{align}}}
and MA latency of the $m$-th tier sub-models is expressed as 
{{ \begin{align}
T_{m,A}\!\left( {\boldsymbol{\mu} } \right) \!={\mathop {\max }\limits_j \left\{ {T_{m}^{j,U}} \right\} + \mathop {\max }\limits_j \left\{ {T_{m}^{j,D}} \right\}}. 
\end{align}}}

Considering the split training is executed per training round and the sub-model aggregation for $m$-tier occurs every $I_m$ training round, the total training latency 
for $R$ training rounds is given by
{{ \begin{align}\label{total_latency}
T\!\left( {{\bf I},\boldsymbol{\mu} } \right) =RT_S\!\left( {\boldsymbol{\mu} } \right) + \sum\limits_{m = 1}^{M - 1} {\left\lfloor {\frac{R}{{{I_m}}}} \right\rfloor T_{m,A}\!\left( {\boldsymbol{\mu} } \right)}. 
\end{align}}}

As alluded in Sec.~\ref{Intro},  MA balances the trade-off between communication overhead and training convergence, while MS significantly impacts communication-computing overhead and model convergence. Therefore, jointly optimizing MA and MS is critical for the performance of HSFL. To this end, we formulate the following optimization problem to minimize the training latency for model convergence:
{\begin{align}\label{time_minimize_problem}
\mathcal{P}:&\mathop {{\rm{min}}}\limits_{{\bf I},{\boldsymbol{\mu}}}  T({\bf I},{\boldsymbol{\mu }})   \\
&\mathrm{s.t.} ~\mathrm{C1:}~\frac{1}{R} \sum_{t=1}^{R} \mathbb{E} [\Vert \nabla_{\bf{w}} f({\mathbf{\overline w}}^{t-1})\Vert^{2}] \le \varepsilon, \nonumber \\
&~\mathrm{C2:}~{\mu _{m,l}} \in \left\{ {0,1} \right\}, \quad \forall m \in {\mathcal{M}}\setminus{\{{M}\}}, l = 1,2,...,L, \nonumber\\
&~\mathrm{C3:}~\sum\limits_{l = 1}^{L} {{\mu _{m,l}}}  = 1, \quad \forall m \in {\mathcal{M}}\setminus{\{{M}\}}, \nonumber \\
&~\mathrm{C4:}~\sum\limits_{l = 1}^\ell  {{\mu _{m,l}}}  \le \sum\limits_{l = 1}^\ell  {{\mu _{m - 1, l}}}, \quad  \ell=1,2,...,L, \nonumber \\
&~\mathrm{C5:}~N^j_m \sum\limits_{l = 1}^L {\left( {{\mu _{m,l}} - {\mu _{m - 1,l}}} \right)} \big( {{{\widetilde \psi} _l} +{{\widetilde \chi} _l}+{{\widetilde \vartheta} _l} + {\delta _l}} \big) < {\rho^j_{m}}, \nonumber \\& \quad \quad \quad \!\!\!  \forall j \in {\mathcal{J}}_m, \forall m \in {\mathcal{M}}\setminus{\{{M}\}}, \nonumber \\
&~\mathrm{C6:}~I_m \in {\mathbb{N}^ + }, \quad \forall m \in {\mathcal{M}}\setminus{\{{M}\}}, \nonumber
\end{align}}
where ${{\widetilde \psi} _l} = \sum\limits_{k = 1}^l {{\psi _k}} $ and ${{\widetilde \chi} _l} = \sum\limits_{k = 1}^l {{{\chi}_k}} $ represent the cumulative sum of data size (in bits) of activations and activations' gradients for first $l$ layers of the neural network, ${{\widetilde \vartheta} _l}$ is the data size of the optimizer state for the first $l$ layers of the neural network, depending on the choice of the optimizer (e.g. SGD, Momentum, and Adam), ${\rho^j_{m}}$ denotes the memory limitation of the $j$-th computing entity in the $m$-th tier. Constrain $\mathrm C1$ guarantees model convergence accuracy; $\mathrm C2$ and $\mathrm C3$ ensure the uniqueness of the cut layer between sub-models of $m$-th and ($m$+1)-th tier; $\mathrm C4$ guarantees that the cut layer of the lower tier is shallower than that of the higher tier; $\mathrm C5$ represents the memory limitation of computing entities~\cite{yeung2021horus}; $\mathrm C6$ denotes that the sub-model MA decision variable is a positive integer.

Problem~\eqref{time_minimize_problem} is a combinatorial optimization problem with a non-convex mixed-integer non-linear objective function. In general, this problem is NP-hard, rendering it infeasible to obtain the optimal solution using polynomial-time algorithms.

\section{Solution Approach}\label{solu_appro}

In this section, we develop an efficient iterative algorithm by decoupling the problem~\eqref{time_minimize_problem} into MA and MS sub-problems and then find the optimal solution for each.

We first reveal the explicit expression $R$ with {\bf{Corollary 1}}. Since $R$ is proportional to the objective function, the objective function is minimized if and only if inequality~\eqref{lowest_com_num} holds as equality. In general, $R$ can be significantly larger than $I_m$, and therefore we can utilize $\left\lfloor {\frac{R}{{{I_m}}}} \right\rfloor  \approx \frac{R}{{{I_m}}}$ to approximate~\eqref{total_latency}. By substituting~\eqref{lowest_com_num} into~\eqref{total_latency}, problem~\eqref{time_minimize_problem} can be converted into 
\begin{align}\label{problem_1}
\mathcal{P'}:&\mathop {{\rm{min}}}\limits_{{\bf I},{\boldsymbol{\mu}}} \Theta ( {\bf I}, {\boldsymbol{\mu }} )   \\
&\mathrm{s.t.} ~\mathrm{C2}-\mathrm{C6}, \nonumber
\end{align}
where 
{ \small \begin{align} \label{objective_1}
\Theta ( {\bf I}, {\boldsymbol{\mu }} )=\frac{{2\vartheta ({T_S}\left( {{\boldsymbol{\mu }}} \right) + \!\sum\limits_{m = 1}^{M - 1} {\frac{{{T_{m,A}}\left( {{\boldsymbol{\mu }}} \right)}}{{{I_m}}}} )}}{{\gamma (\varepsilon  \!- \!\frac{{\beta \gamma \sum\limits_{l = 1}^L {\sigma _l^2} }}{N} - 4{\beta ^2}{\gamma ^2}\!\sum\limits_{m = 1}^{M - 1}  ({\mathbbm{1}}_{\{I_m > 1\}} I_m^2\!\!\sum\limits_{l = {L_{m - 1}} + 1}^{{L_{m - 1}} + {L_m}} \!\!\! G_l^2))}}.
\end{align}}

The term $\sum\limits_{l = {L_{m - 1}} + 1}^{{L_{m - 1}} + {L_m}} {G_l^2} $ is intertwined with MS, but it does not disclose the relationship with MS decision variables. To address this issue, we introduce a set of constants ${\bf{\widetilde G}} = \left[ {{{\widetilde G}^2_1},{{\widetilde G}^2_2},...,{{\widetilde G}^2_L}} \right]$, where ${\widetilde G}^2_l$ represents the cumulative sum of the bounded second order moments for the first $l$ layers of neural network, defined as ${{\widetilde G}^2_l} = \sum\limits_{k = 1}^l {G_k^2}$. Hence, $\sum\limits_{l = {L_{m - 1}} + 1}^{{L_{m - 1}} + {L_m}} {G_l^2} $ can be reformulated as $\sum\limits_{l = 1}^L {\left( {{\mu _{m,l}} - {\mu _{m - 1,l}}} \right)} {{\widetilde G}_l^2}$. Moreover, the non-convexity and non-smoothness of the objective function of problem~\eqref{problem_1} make it extremely intractable. To linearize the objective function, we introduce a set of auxiliary variables ${\bf{T}} = \left[ {{T_1},{T_{1,2}},T_{2,2}, ..., T_{M-1,2}, {T_{1,3}},T_{2,3}, ..., T_{M-1,3}} \right]$, i.e., $\mathop {\max }\limits_n \{\! \sum\limits_{m = 1}^M {T_{m,n}^F + \sum\limits_{m = 1}^{M - 1} {T_{m,n}^A + \sum\limits_{m = 1}^M {T_{m,n}^B + \sum\limits_{m = 1}^{M - 1} {T_{m,n}^G} } } } \}  \le {T_1}$, ${\mathop {\max }\limits_j \left\{ {T_{m}^{j,U}} \right\}} \le {T_{m,2}}$, and ${\mathop {\max }\limits_j \left\{ {T_{m}^{j,D}} \right\}} \le {T_{m,3}}$. Therefore, problem~\eqref{problem_1} can be transformed into
{\small \begin{align}\label{problem_2}
\mathcal{P''}:&\mathop {{\rm{min}}}\limits_{{\bf I}, {\boldsymbol{\mu}}, \bf{T} } \Theta' ( {\bf I}, {\boldsymbol{\mu }}, \bf{T} )   \\
&\mathrm{s.t.} ~\mathrm{C2}-\mathrm{C6}, \nonumber\\
&~\mathrm{R1:}~\!\!\!\sum\limits_{m = 1}^M \!\frac{{\!\sum\limits_{l = 1}^L \!\!{\left( {{\mu _{m,l}}\! -\! {\mu _{m - 1,l}}} \right)} \!\left( {{\rho _l(b)} + {\varpi _l(b)}} \right)}}{{{f_{m,n}}}} \!+\!\! \sum\limits_{m = 1}^{M - 1} \!\frac{{b\!\!\sum\limits_{l = 1}^L \!{{\mu _{m,l}}} {\psi _l}}}{{r_{m,n}^A}} \nonumber \\&\quad \quad \quad + \sum\limits_{m = 1}^{M - 1} {\frac{{b\sum\limits_{l = 1}^L {{\mu _{m,l}}} {\chi _l}}}{{r_{m,n}^G}} \le {T_1}},    \quad\forall n \in \mathcal{N},  \nonumber\\
&~\mathrm{R2:}~\!{\mathbbm{1}}_{\{J_m > 1\}} \frac{{\sum\limits_{l = 1}^L \!\!{\left( {{\mu _{m,l}}\! -\! {\mu _{m - 1,l}}} \right)} {\delta _l}}}{{r_{m}^{j,U}}}\! \le\! {T_{m,2}}, \; \;\forall j \!\in {\mathcal{J}}_m, \nonumber  \forall m \in \\ &\quad \quad \quad {\mathcal{M}}\setminus{\{{M}\}}, \nonumber\\
&~\mathrm{R3:}~\!{\mathbbm{1}}_{\{J_m > 1\}} \frac{{\sum\limits_{l = 1}^L \!\!{\left( {{\mu _{m,l}}\! -\! {\mu _{m - 1,l}}} \right)} {\delta _l}}}{{r_{m}^{j, D}}}\! \le\! {T_{m,3}}, \; \;\forall j \!\in {\mathcal{J}}_m, \forall m \in \nonumber \\ &\quad \quad \quad {\mathcal{M}}\setminus{\{{M}\}},  \nonumber
\end{align}}
where
{ \scriptsize \begin{align}
\Theta' ( {\bf I}, {\boldsymbol{\mu }}, \!{\bf{T}} )\! = \!\frac{{2\vartheta ({T_1} + \! \sum\limits_{m = 1}^{M - 1} {\frac{{{T_{m,2}} + {T_{m,3}}}}{{{I_m}}}} )}}{{\gamma (\varepsilon \! - \!\frac{{\beta \gamma \!\! \sum\limits_{l = 1}^L \!\!{\sigma _l^2} }}{N} \!- 4{\beta ^2}{\gamma ^2}\!\!\!\sum\limits_{m = 1}^{M - 1} \!\!( {\mathbbm{1}}_{\{I_m > 1\}} I_m^2\!\!\sum\limits_{l = 1}^L\!\! {\left( {{\mu _{m,l}} \!-\! {\mu _{m - 1,l}}} \right)} {\widetilde G}_l^2))}}. 
\end{align}}
 
The difficulty in solving problem~\eqref{problem_2} stems from the tight coupling between auxiliary variables and the original decision variables. Therefore, we decompose the problem~\eqref{problem_2} into two tractable sub-problems based on decision variables and develop efficient algorithms to obtain optimal solutions for each.

We fix the variables $\boldsymbol{\mu}$ and $\bf{T}$ to investigate the sub-problem involving sub-model MA, which is expressed as
\begin{align}\label{subproblem_1}
\mathcal{P}_1:&\mathop {{\rm{min}}}\limits_{{\bf I}} \Theta' ( {\bf I} )  \\
&\mathrm{s.t.} ~\mathrm{C6}. \nonumber
\end{align}

Then, we can derive the following proposition: 
\begin{proposition}\label{theorem2}
By fixing the sub-model MA intervals of tiers ${\mathcal{M}}' \subset {\mathcal{M}}\setminus{\{{M}\}}$ to 1, the optimal sub-model MA intervals (assuming their sub-model MA intervals are larger than 1) for remaining tiers ${\mathcal{M}}'' = {\mathcal{M}}\setminus{\left( {\{ M\}  \cup {\mathcal{M}}'} \right)}$ are given by
{  \begin{align}\label{accuracy_cons}
{{\bf{I}}^*} =  {{\mathrm{argmin}} _{\scriptstyle{\bf{I}} \in \{ \left. {I_{m}} \right|\hfill
\scriptstyle{I_{m}} \in \{ \max (\left\lfloor {{{\hat I}_{m}}} \right\rfloor, \left\lceil {{{\hat I}_{m}}} \right\rceil )\} \} \hfill}}\Theta '\!\!\left( {\bf{I}} \right),
\end{align}}
where  ${\bf{\hat I}} = \{ {{\hat I}_m}\mid m \in {\cal M}''\} $ can be easily obtained by solving $\frac{{\partial \Theta '({\bf{I}})}}{{\partial {\bf{I}}}} = {\bf{0}}$ with Newton-Jacobi  method~\cite{geng2019pipeline}.
\end{proposition}

\begin{proof}
See Appendix C.
\end{proof}
{\textbf{Insights:}} Eqn.~\eqref{E_I} demonstrates that tiers with lower communication overhead for sub-model aggregation (i.e., smaller $b_m$) tend to aggregate sub-models more frequently (i.e., smaller $I_m$), otherwise aggregate models less often. This is because, as indicated by Eqn.~\eqref{lowest_com_num}, 
the slower model convergence caused by infrequent aggregation in one tier can be compensated by shortening the aggregation intervals in other tiers. Therefore, to achieve the target accuracy within the shortest training time, tiers with lower communication overhead for sub-model aggregation should utilize shorter aggregation intervals. In the HSFL system, the fed server typically acts as an application server responsible for model aggregation across a relatively large geographical area, often operating as a cloud server over the Internet. Consequently, there often exists ${b_{m_1}} > {b_{m_2}} > ... > {b_{{m_{|{\cal M}''|}}}}$ with with layer index $m_1 > m_2 > ... > |{\cal M}''|$, yielding ${I_{m_1}} > {I_{m_2}} > ... > {I_{{m_{|{\cal M}''|}}}}$.

To circumvent the difficulty caused by the indicator function ${\mathbbm{1}}_{\{\cdot\}}$, \textbf{Proposition 1} provides the optimal solutions by assuming $I_m>1$. Based on {\bf Proposition 1}, the optimal solution to problem~\eqref{subproblem_1} can be found by exhaustively evaluating all possible combinations by fixing some $I_m$ to $1$. Specifically, the sub-model MA interval for each tier has two potential solutions (i.e., 1 or larger than 1), resulting in $2^M$ combinations. The objective function is calculated for each combination, and the one with the minimum value is identified as the global optimum. In practice, $M$ is typically small (e.g., $M = 3$ in the classical client-edge-cloud architecture), ensuring the computational complexity remains manageable. 

By fixing the decision variable $\bf I$, we transform problem~\eqref{problem_2} into a standard mixed-integer linear fractional  programming (MILFP) with respect to $\boldsymbol{\mu}$ and ${\bf{T}}$, which is given by
\begin{align}\label{subproblem_2}
\mathcal{P}_2:&\mathop {{\rm{min}}}\limits_{{\boldsymbol{\mu}, {\bf{T}}}} \Theta' ( {\boldsymbol{\mu }}, {\bf{T}} )   \\
&\mathrm{s.t.} ~\mathrm{C2}-\mathrm{C5},~\mathrm{R1}-\mathrm{R3}. \nonumber
\end{align}

We leverage the Dinkelbach algorithm~\cite{dinkelbach1967nonlinear} to optimally solve problem~\eqref{subproblem_2} by introducing the fractional parameter to reformulate it as mixed-integer linear programming (MILP)~\cite{yue2013reformulation,rodenas1999extensions}.

As aforementioned, we decompose the original problem~\eqref{time_minimize_problem} into two tractable MA and MS sub-problems and develop efficient algorithms to find optimal solutions for each. Subsequently, we propose a block-coordinate descent (BCD)-based algorithm~\cite{tseng2001convergence} to solve the original problem~\eqref{time_minimize_problem}, as outlined in~\textbf{Algorithm~\ref{BCD-based}}. It is noted that estimation of the key parameters for executing the algorithm (e.g., $\beta$, ${G _l^2}$ and ${\sigma _l^2}$) follows the approach in~\cite{wang2019adaptive}. The algorithm can be re-executed after a certain period of time, particularly when the key parameters and network conditions change significantly.

\section{Implementation}\label{sec:implementa}
In this section, we elaborate on the implementation of HSFL in terms of simulation setup, dataset and model, and benchmarks.

{\bf Experimental setup:} In our experiment, we consider a classical client-edge-cloud three-tier HSFL ($M=3$) consisting of a cloud server, $5$ edge servers, and $20$ edge devices ($J_1=N=20, J_2=5,$ and $ J_3=1$), with each edge server connected to $4$ edge devices. The computing capabilities of the cloud server and edge server are set to $50$ TFLOPS and $5$ TFLOPS, and the computing capability of each edge device is uniformly distributed within $[0.4, 0.6]$ TFLOPS. The uplink transmission rates from edge devices to edge servers and the fed server follow uniform distribution within $[75, 80]$ Mbps, and the corresponding downlink rates are set to $370$ Mbps. The transmission rates between edge servers and the cloud server and between edge servers and the fed server follow a uniform distribution within $[370, 400]$ Mbps. We set mini-batch size and learning rate to $16$ and $5\times {10^{-4}}$, respectively. It is assumed that each computing entity evenly allocates communication and computing resources to each of its sub-models.



\begin{algorithm}[t]
    \small
	\renewcommand{\algorithmicrequire}{\textbf{Input:}}
	\renewcommand{\algorithmicensure}{\textbf{Output:}}
	\caption{BCD-based Algorithm.}\label{BCD-based}
	\begin{algorithmic}[1]
        \REQUIRE  ${\bf I}^{(0)}$, ${\boldsymbol{\mu }}^{(0)}$, ${{\bf{T}}}^{(0)}$ and convergence threshold ${\varepsilon }$.
		\ENSURE  ${{\bf I}^{\bf{*}}}$ and ${{\boldsymbol{\mu }}^{\bf{*}}}$  
          \STATE Initialization:  $\tau  \leftarrow 0$.
          \REPEAT 
          \STATE$\tau  \leftarrow \tau+1$
           \STATE Update ${\bf I}^{(\tau)}$ by solving problem~\eqref{subproblem_1}
           \STATE Update ${\boldsymbol{\mu }}^{(\tau)}$ and  ${{\bf{T}}}^{(\tau)}$ by solving problem~\eqref{subproblem_2} 
         \UNTIL{\small{$|\Theta' ( {\bf I}^{(\tau)}, {{\boldsymbol{\mu }}^{(\tau)}}, {{\bf{T}}^{(\tau)}} ) - \Theta' ( {\bf I}^{(\tau-1)}, {\boldsymbol{\mu }}^{(\tau-1)}, {{\bf{T}}^{(\tau-1)}}| \!\le \!{\varepsilon }$}}    
         
	\end{algorithmic}  
\end{algorithm}

{\bf Dataset and model:} We adopt the widely-used image classification datasets CIFAR-10~\cite{krizhevsky2009learning} and MNIST~\cite{lecun1998mnist}. 
The CIFAR-10 dataset is split into 50000 training and 10000 test object images across 10 categories, such as airplanes and ships, and MNIST dataset comprises 60000 training and 10000 test grayscale images of handwritten digits 0 to 9. 
We conduct experiments under IID and non-IID settings. The data samples are shuffled and evenly distributed to all edge devices in the IID setting. In the non-IID setting~\cite{zhu2019broadband,yang2020energy}, we sort the data samples by labels, divide them into 40 shards, and assign 2 shards to each of the 20 edge devices. We employ the well-known VGG-16 neural network~\cite{simonyan2014very}, consisting of 13 convolution layers and 3 fully connected layers, in HSFL.

{\bf Benchmarks:} To comprehensively evaluate the performance of HSFL, we compare HSFL against the
following alternatives:

\begin{itemize}
    \item {\bf{RMA+MS:}} RMA+MS employs a random sub-model MA strategy (i.e., randomly drawing sub-model MA $I_m$ from 1 to 25 during model training.), and adopts the tailored MS scheme in Section~\ref{solu_appro}.    
    \item {\bf{MA+RMS:}} MA+RMS utilizes the sub-model MA strategy in Section~\ref{solu_appro} and employs a random MS scheme (i.e., model split points are randomly selected from 3 to 14 during model training).
    \item {\bf{RMA+RMS:}} RMA+RMS adopts the random sub-model MA and MS strategy.
    \item {\bf{DAMA+RMS:}} The DAMA+RMS benchmark employs the depth-aware adaptive client-side MA scheme~\cite{you2023aifed} and utilizes random MS strategy.
    \item {\bf{RMA+AMS:}} The RMA+HMS benchmark employs the random client-side MA scheme and utilizes resource-heterogeneity-aware MS strategy~\cite{wang2023coopfl}.
\end{itemize}

\section{Performance Evaluation}\label{simu_results}

In this section, we evaluate the performance of HSFL 
from three aspects: i) comparisons with five benchmarks to demonstrate the superiority of HSFL; ii) investigating the robustness of HSFL to
varying network computing and communication resources; iii) ablation study to show the necessity of each meticulously designed component in HSFL, including model aggregation (MA) and model splitting (MS).




\begin{figure}[t]
    \setlength\abovecaptionskip{6pt}
     \setlength\subfigcapskip{0pt}
     \centering
\subfigure[CIFAR-10 under IID setting.]{
\includegraphics[width=.434\columnwidth]{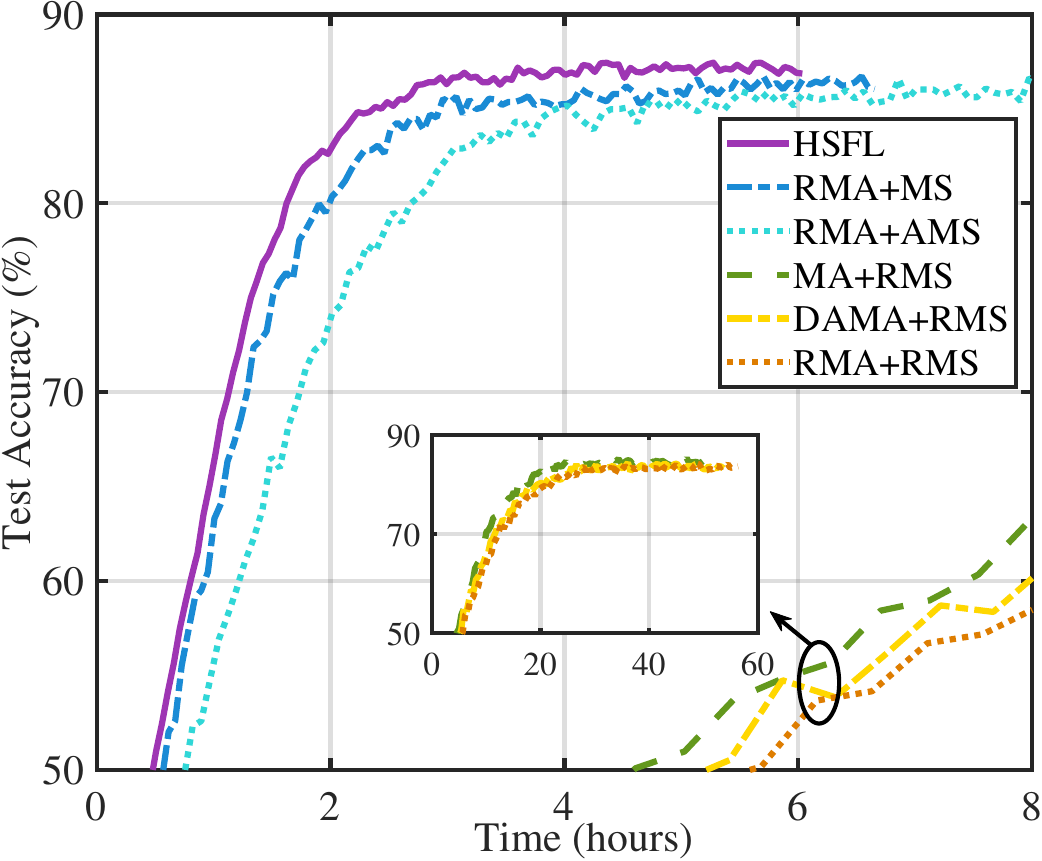}
    \label{sfig:cifar_iid_test_accuracy}
}
\subfigure[CIFAR-10 under non-IID setting.]{
    \includegraphics[width=.437\columnwidth]{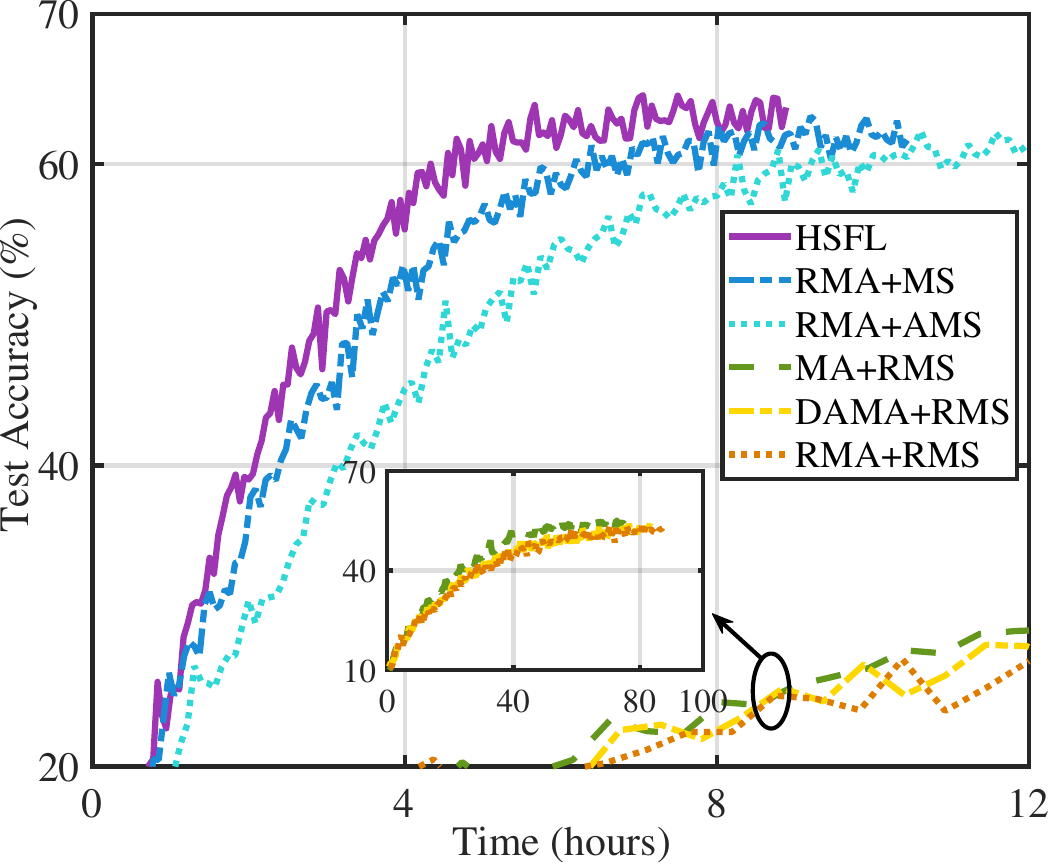}
    \label{sfig:cifar_non_iid_test_accuracy}
}
\subfigure[MNIST under IID setting.]{
\includegraphics[width=.434\columnwidth]{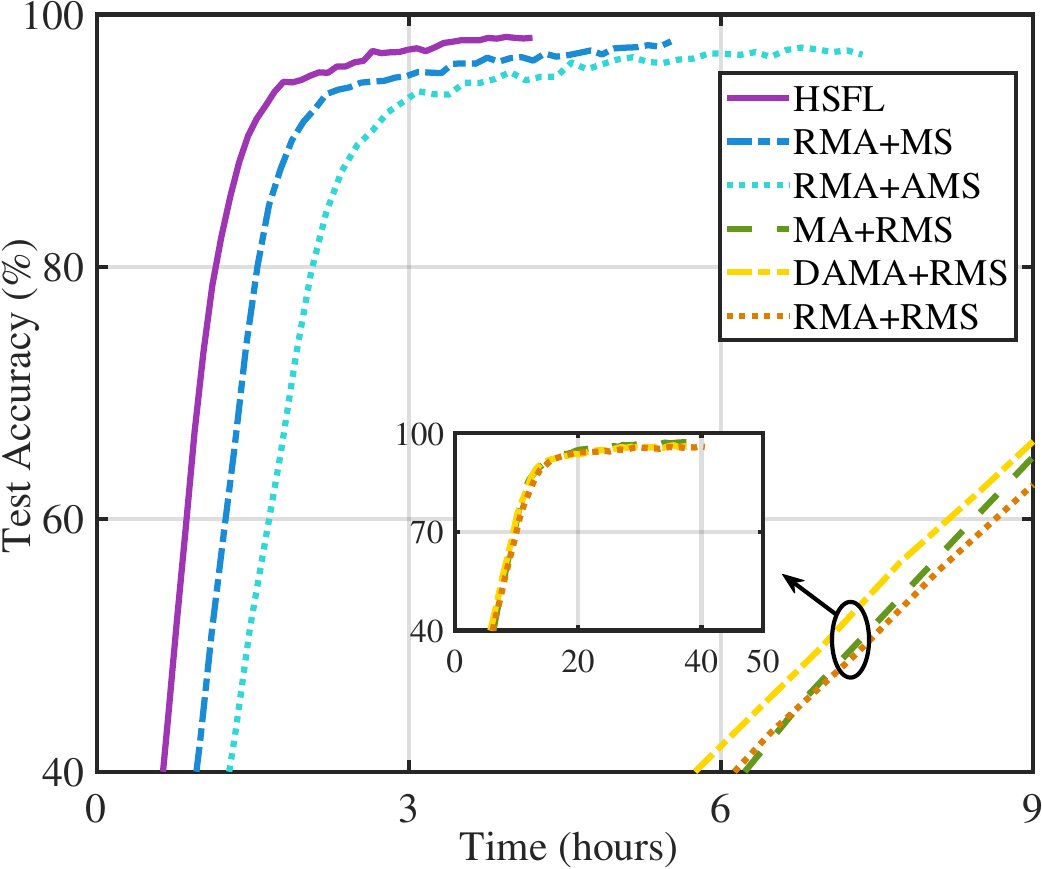}
    \label{sfig:mnist_iid_test_accuracy}
}
\subfigure[MNIST under non-IID setting.]{
    \includegraphics[width=.439\columnwidth]{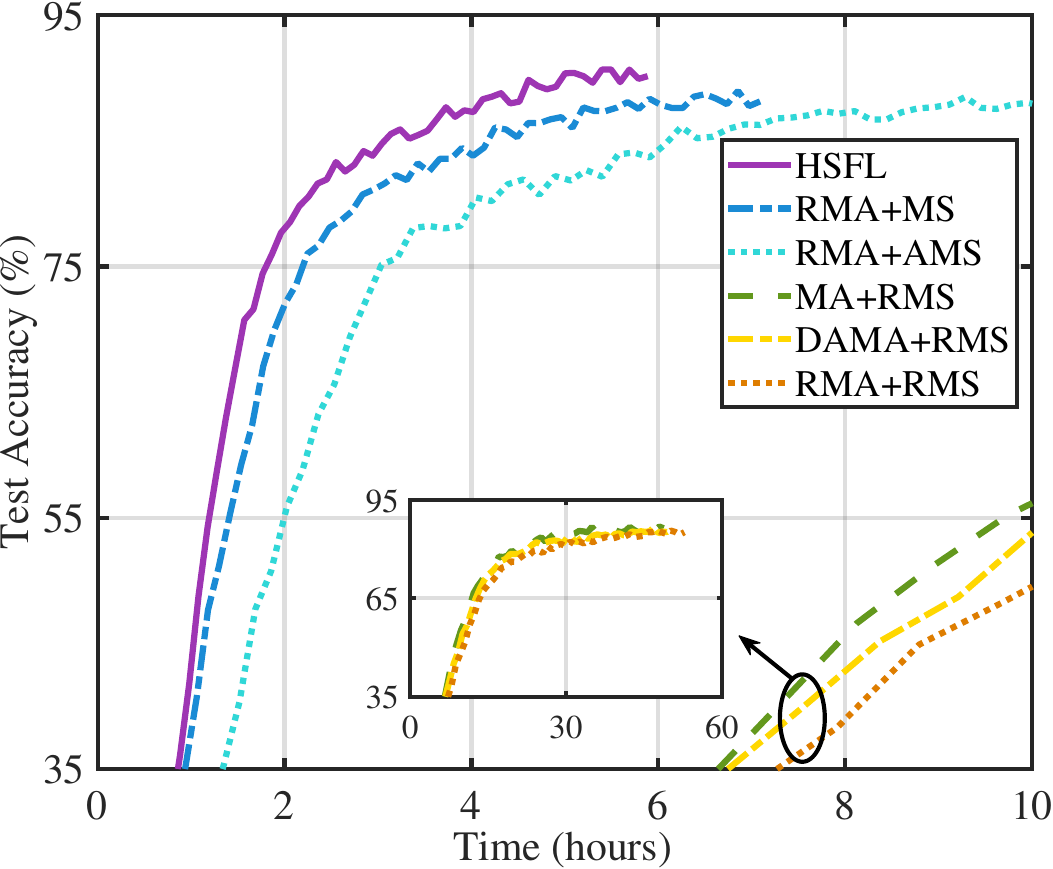}
    \label{sfig:mnist_non_iid_test_accuracy}
}
    \caption{The training performance on CIFAR-10 and MNIST datasets under IID and non-IID settings using VGG-16.}
    \label{fig:test_accuracy}
\end{figure}

\subsection{Superiority of HSFL}\label{simu_setup}

In this section, we conduct a comprehensive comparison of HSFL against five benchmarks in terms of test accuracy and
convergence speed.




Fig.~\ref{fig:test_accuracy} shows the training performance of HSFL and five benchmarks on CIFAR-10 and MNIST datasets. It is seen that HSFL outperforms the other five benchmarks in test accuracy and convergence speed as the model converges.
Notably, HSFL, RMA+MS, and RMA+AMS converge significantly faster than MA+RMS, DAMA+RMS, and RMA+RMS owing to the adaptive MS scheme, which strikes a good balance between communication-computing overhead and training convergence speed to expedite model training. Moreover, HSFL and MA+RMA not only converge faster than RMA+MS and RMA+RMS but also achieve comparable accuracy, thus demonstrating the effectiveness of tailored sub-model MA strategy.  Comparing Fig.~\ref{sfig:cifar_iid_test_accuracy} with Fig.~\ref{sfig:cifar_non_iid_test_accuracy}, and Fig.~\ref{sfig:mnist_iid_test_accuracy} with Fig.~\ref{sfig:mnist_non_iid_test_accuracy}, shows that the convergence speed of HSFL and other five benchmarks are slower under non-IID setting than under IID setting.

Fig.~\ref{fig:time_accuracy} presents converged accuracy and time (i.e., the incremental increase in test accuracy is less than 0.02$\%$) of HSFL and five benchmarks on CIFAR-10 and MNIST datasets. The performance gap between HSFL and RMA+MS and MA+RMS benchmarks reveals the substantial impact of MA and MS on training performance. 
Moreover, the impact of MS on model training outweighs that of MA frequency. This is because model split points directly determine the overall aggregation interval of the global model, thereby affecting the effectiveness of model training. In the IID setting, HSFL exhibits approximately {{$2.1 \%$}} and {{$1 \%$}} increase in accuracy and nearly $8.1$ and $8.5$ folds acceleration in model convergence over its counterparts without MS optimization (MA+RMS) on the CIFAR-10 and MNIST datasets, respectively. Moreover, this performance improvement is more pronounced under the non-IID setting, reaching about {{$8.8 \%$}} and {{$4 \%$}} in converged accuracy and $8$ and $7$ folds in convergence speed.  RMA+AMS and DAMA+RMS exhibit slower model convergence and lower accuracy compared to RMA+MS and MA+RMS, primarily because they are not designed based on model convergence. The comparison between HSFL and RMA+RMS reveals that HSFL converges at least $9$ times faster than its unoptimized counterparts with guaranteed training accuracy, underscoring the superior performance of the HSFL framework.

\begin{figure}[t]
    \setlength\abovecaptionskip{6pt}
     \setlength\subfigcapskip{0pt}
     \centering
\subfigure[CIFAR-10 under IID setting.]{
\includegraphics[width=.435\columnwidth]{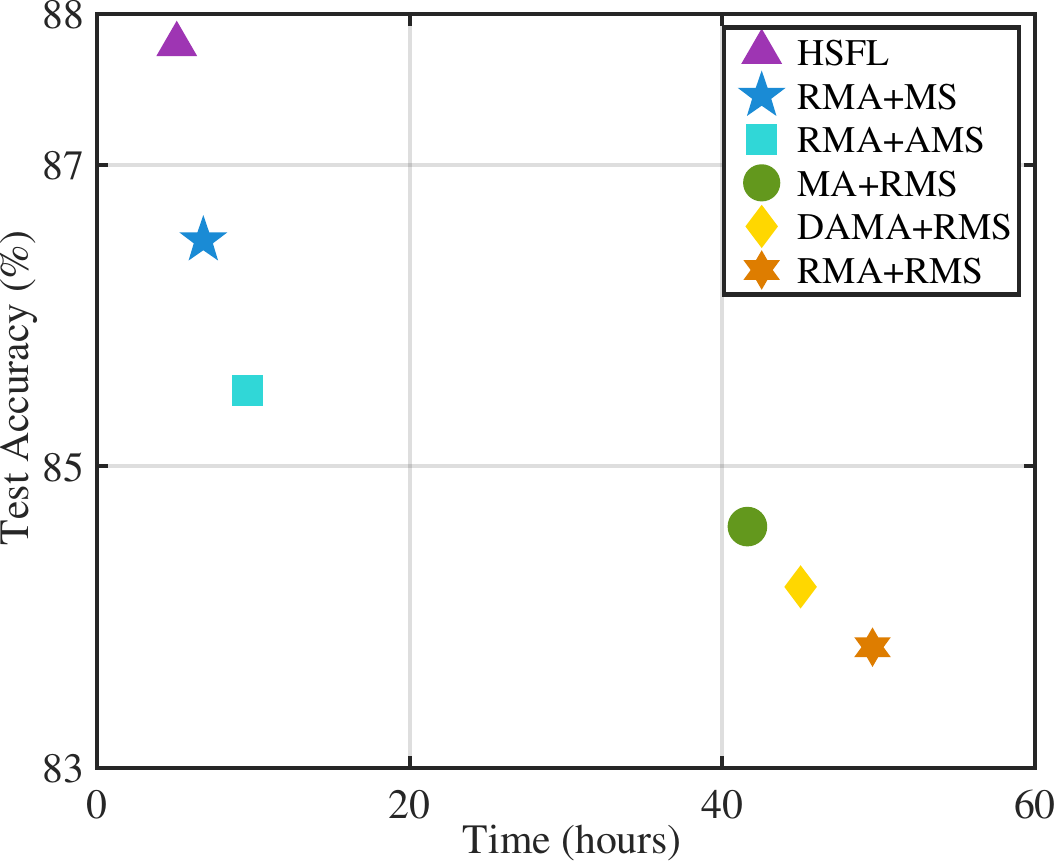}
    \label{sfig:cifar_iid_time_accuracy}
}
\subfigure[CIFAR-10 under non-IID setting.]{
    \includegraphics[width=.436\columnwidth]{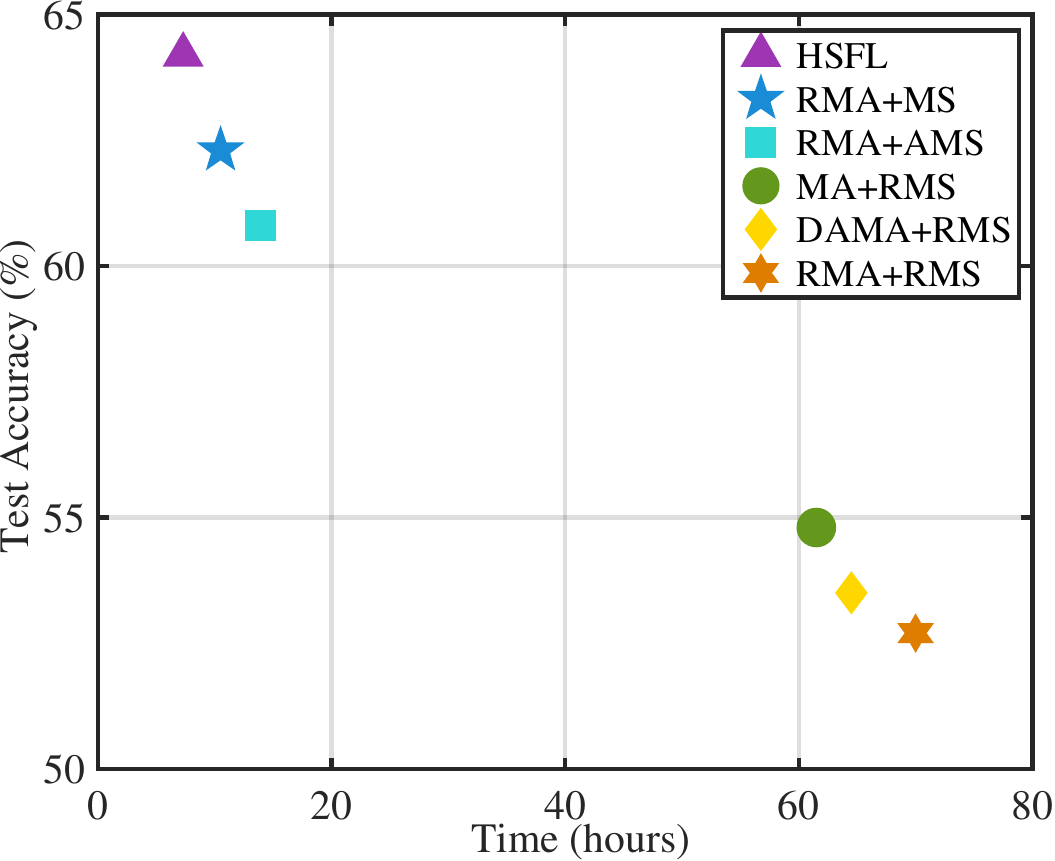}
    \label{sfig:cifar_non_iid_time_accuracy}
}
\subfigure[MNIST under IID setting.]{
\includegraphics[width=.435\columnwidth]{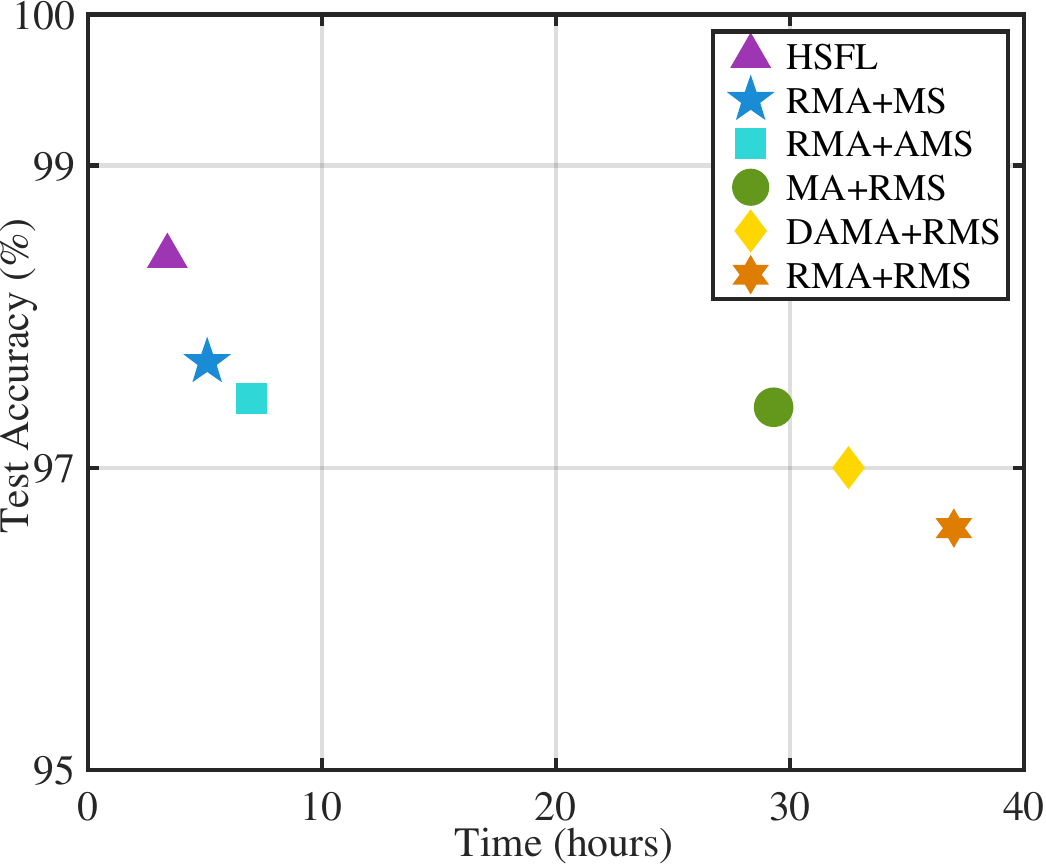}
    \label{sfig:mnist_iid_time_accuracy}
}
\subfigure[MNIST under non-IID setting.]{
    \includegraphics[width=.436\columnwidth]{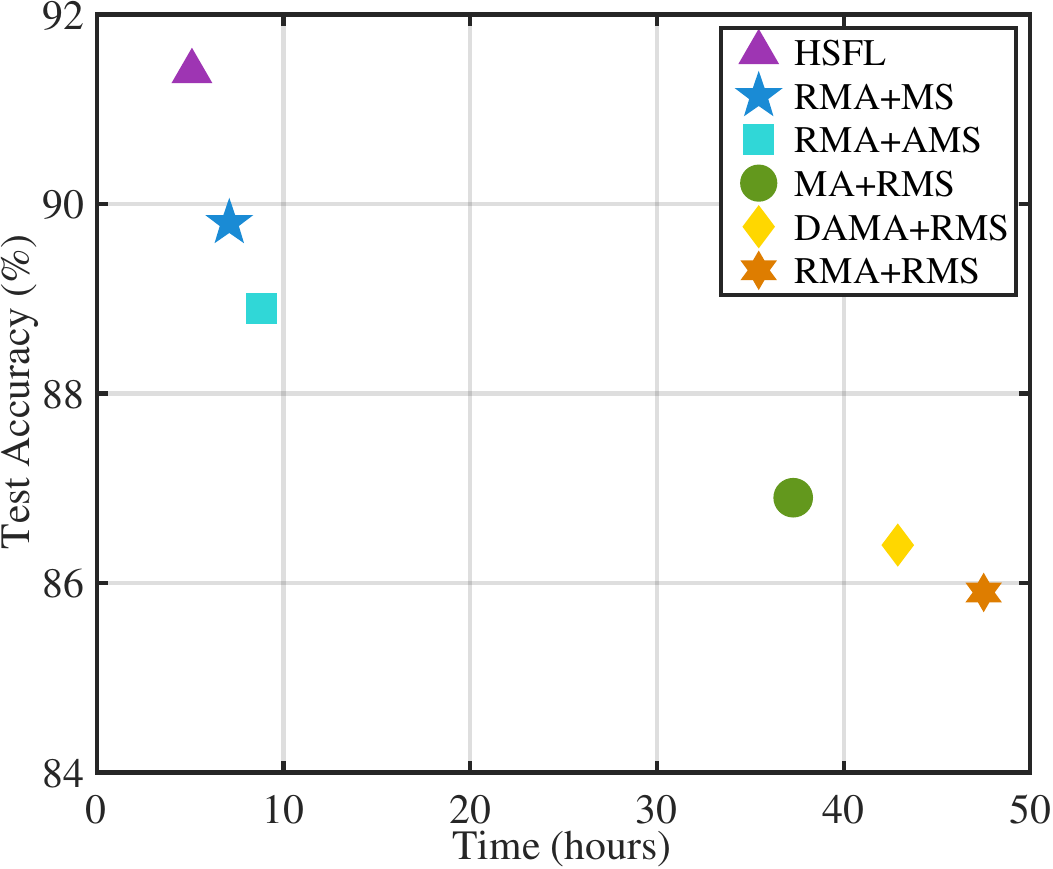}
    \label{sfig:mnist_non_iid_time_accuracy}
}
    \caption{The converged test accuracy and time on CIFAR-10 and MNIST datasets under IID and non-IID settings using VGG-16.}
    \label{fig:time_accuracy}
\end{figure}

\subsection{The Impact of Varying Network Resources}
In this section, we investigate the robustness of HSFL to varying network computing and communication resources. 

Fig.~\ref{fig:comput_accuracy} illustrates the converged time versus network computing and communication resources on the CIFAR-10 dataset under the IID setting. HSFL exhibits a faster convergence speed than the other five benchmarks across varying network resources.  It is clear that the convergence speed of RMA+RMS notably slows down as network resources diminish. This is because random MS and MA strategy fails to strike a good balance between computing-communication overhead and training convergence to expedite model training. The convergence speed of RMA+MS and MA+RMS decreases more slowly than RMA+RMS, indicating that optimizing either MA or MS can partially mitigate the rise in converged time caused by reduced network resources. In contrast, HSFL's converged time only experiences a slight increase with shrinking network resources. The underlying reason is two-fold: one is that the MS scheme is capable of selecting optimal model split points based on network resource conditions to balance computing-communication overhead and training convergence, and the other is that the MA strategy can adjust sub-model aggregation intervals to achieve the minimum communication rounds required for model convergence. This demonstrates the robustness of HSFL and highlights the adaptability of MA and MS strategies to changes in network resources.

Fig.~\ref{fig:different_SFL} shows the converged time of the HSFL framework with varying tiers under diverse computing and communication resources. Fig~\ref{sfig:computing_different_SFL} presents that the increase in training latency of the three-tier HSFL is negligible compared to client-edge and client-cloud SFL as computing resources decline. This is primarily because the three-tier HSFL incorporates more powerful cloud servers into the model training process, and MS strategy places a larger portion of the model on a cloud server to combat the slowdown in model convergence. Similarly, three-tier HSFL exhibits better robustness to communication resources than client-edge and client-cloud SFL. The reason for this is that more tiers in HSFL provide greater flexibility in optimizing MA and thus improve the effectiveness of model training.

\begin{figure}[t]
    \setlength\abovecaptionskip{6pt}
     \setlength\subfigcapskip{0pt}
     \centering
\subfigure[Computing.]{
\includegraphics[width=.434\columnwidth]{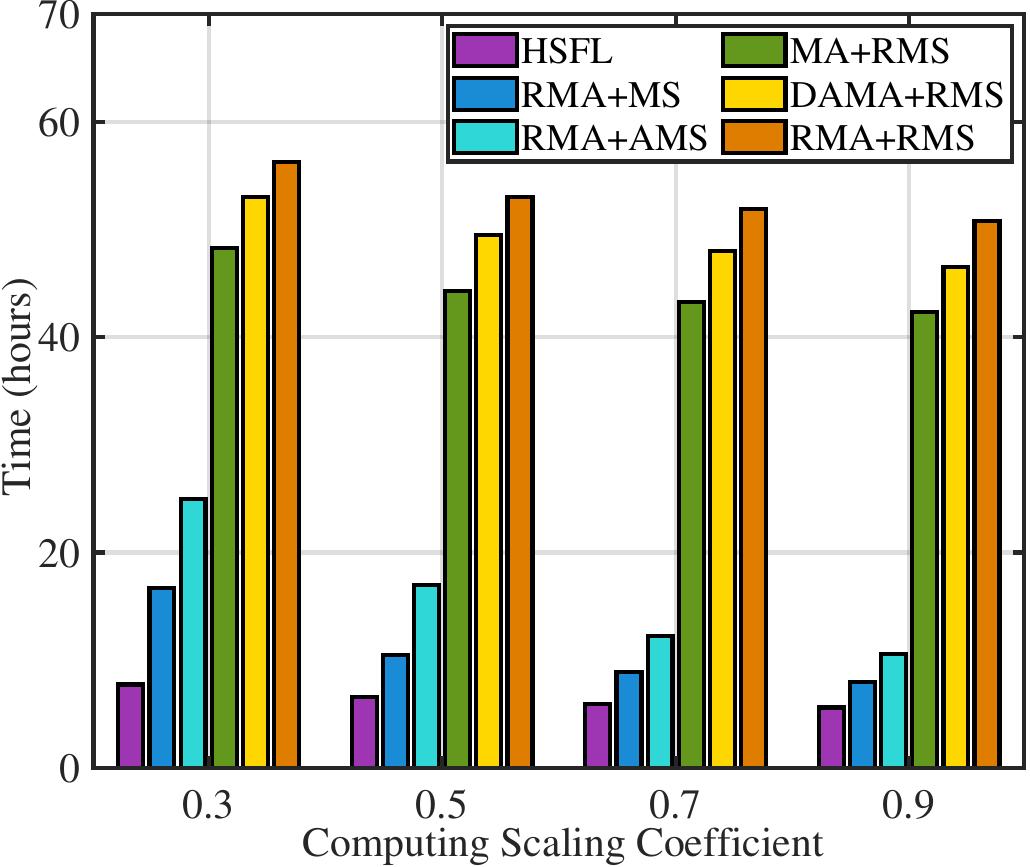}
    \label{sfig:device_comput_accuracy}
}
\subfigure[Communication.]{
    \includegraphics[width=.440\columnwidth]{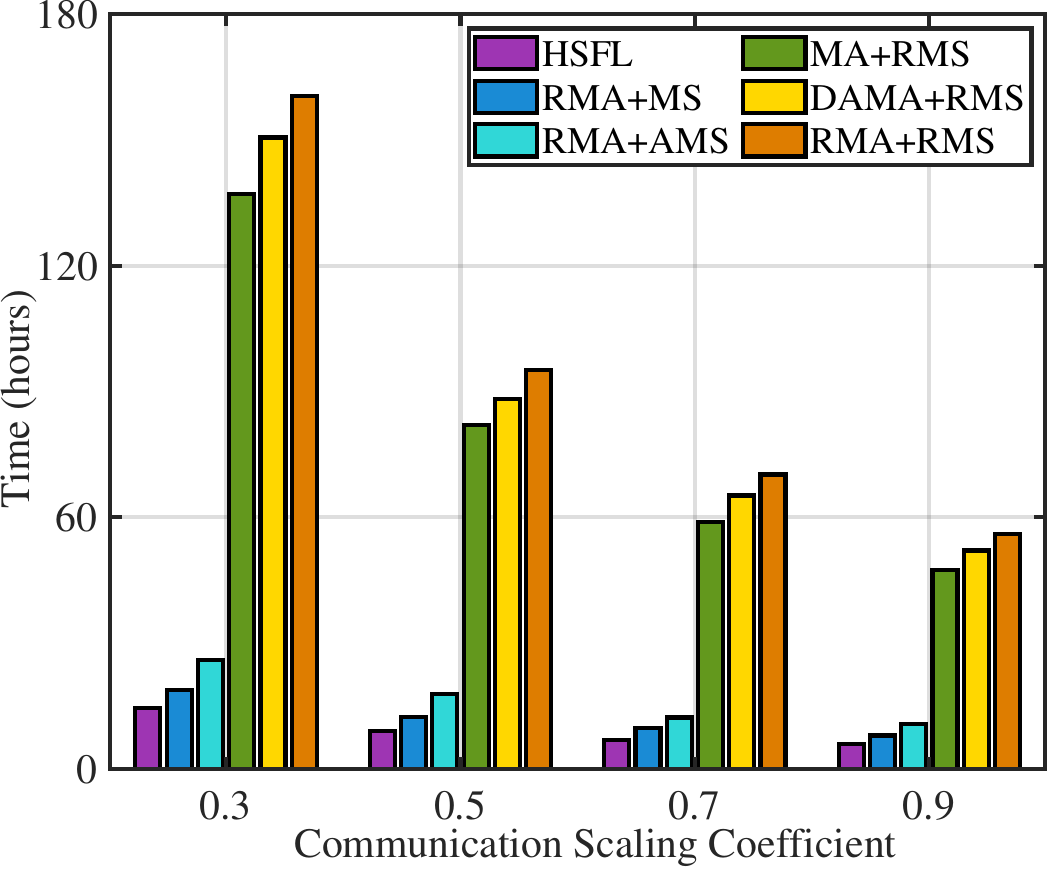}
    \label{sfig:server_comput_accuracy}
}
    \caption{ The converged time versus network computing and communication resources on CIFAR-10 dataset under IID setting, where computing/communication scaling coefficient denotes scaling multiple of computing capabilities of entities/transmission rate, with lower scaling coefficient implying more limited network computing and communication resources.}
    \label{fig:comput_accuracy}
\end{figure}

\begin{figure}[t]
    \setlength\abovecaptionskip{6pt}
     \setlength\subfigcapskip{0pt}
     \centering
\subfigure[Computing.]{
\includegraphics[width=.435\columnwidth]{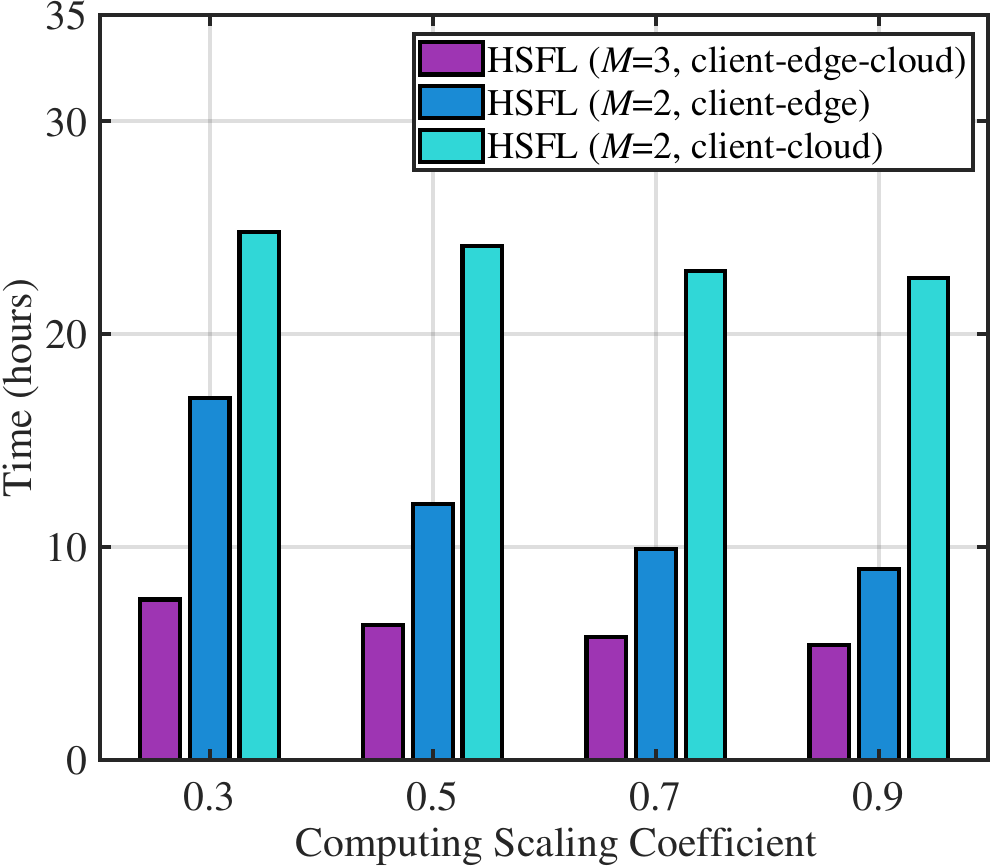}
    \label{sfig:computing_different_SFL}
}
\subfigure[Communications.]{
    \includegraphics[width=.436\columnwidth]{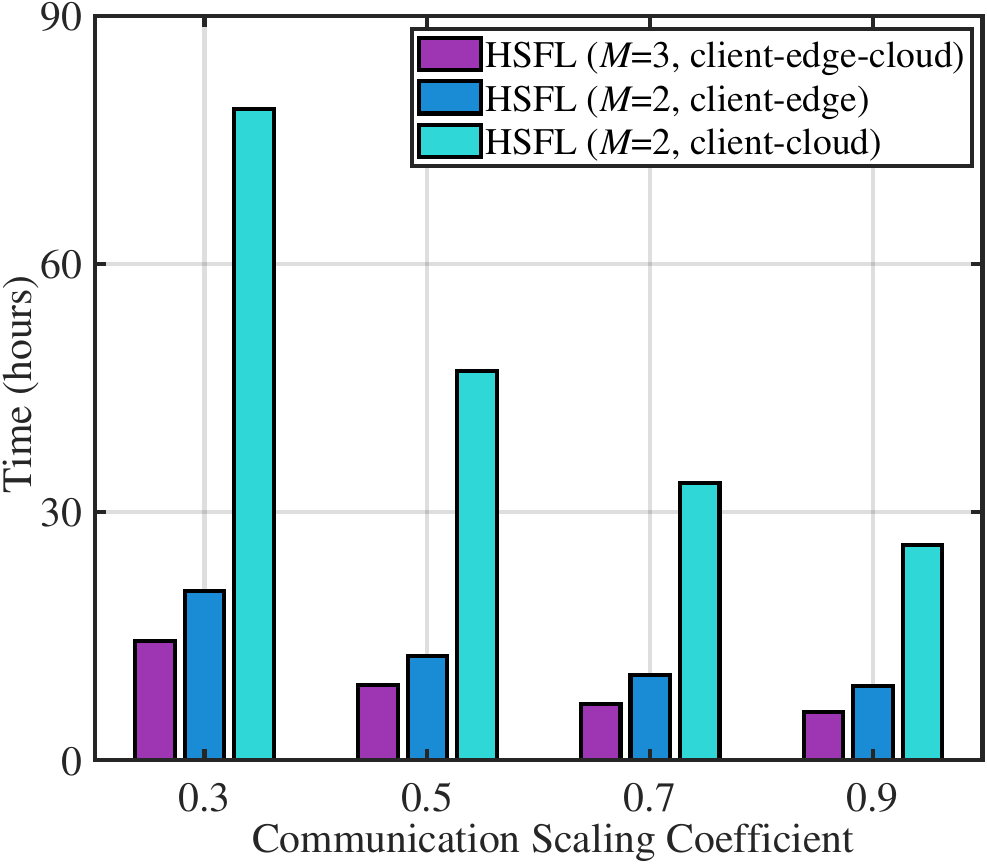}
    \label{sfig:communication_different_SFL}
}
    \caption{The converged time of the HSFL framework with varying tiers under diverse computing and communication resources on CIFAR-10 dataset under IID setting.}
    \label{fig:different_SFL}
\end{figure}

\subsection{Ablation Study of HSFL Framework}

In this section,  we conduct ablation experiments to demonstrate the effectiveness of each component in HSFL.


\begin{figure}[t]
    \setlength\abovecaptionskip{6pt}
     \setlength\subfigcapskip{0pt}
     \centering
\subfigure[CIFAR-10 under IID setting.]{
\includegraphics[width=.434\columnwidth]{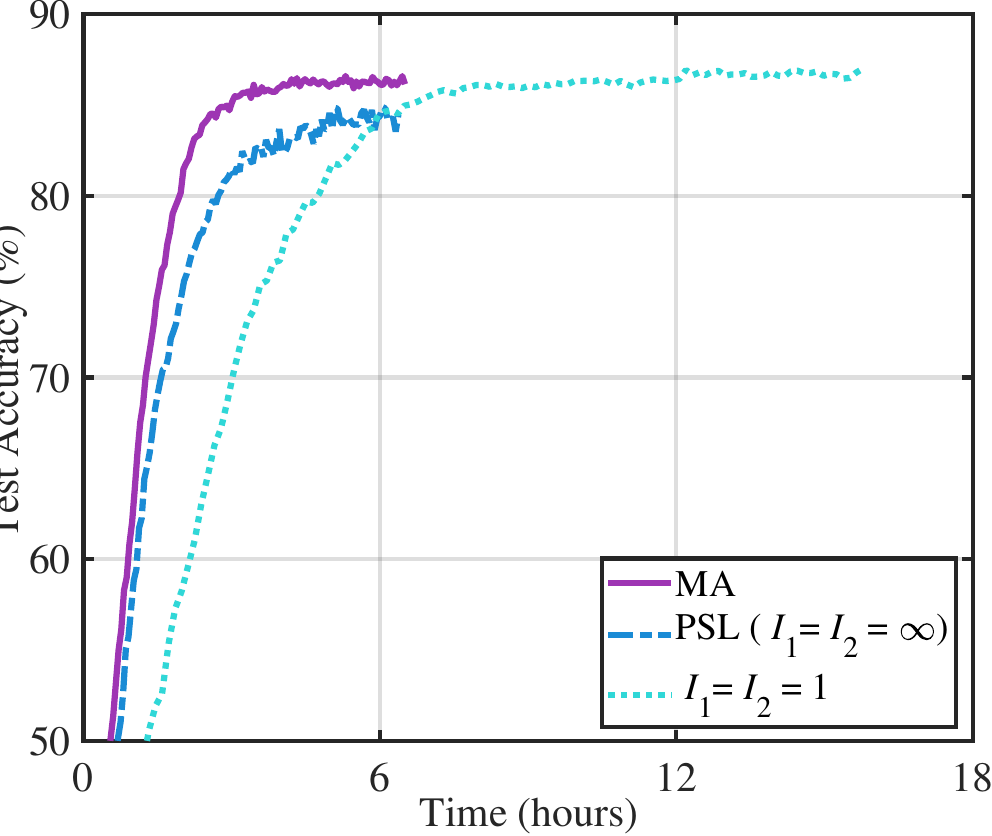}
    \label{sfig:cifar_iid_cut_aba}
}
\subfigure[CIFAR-10 under non-IID setting.]{
    \includegraphics[width=.441\columnwidth]{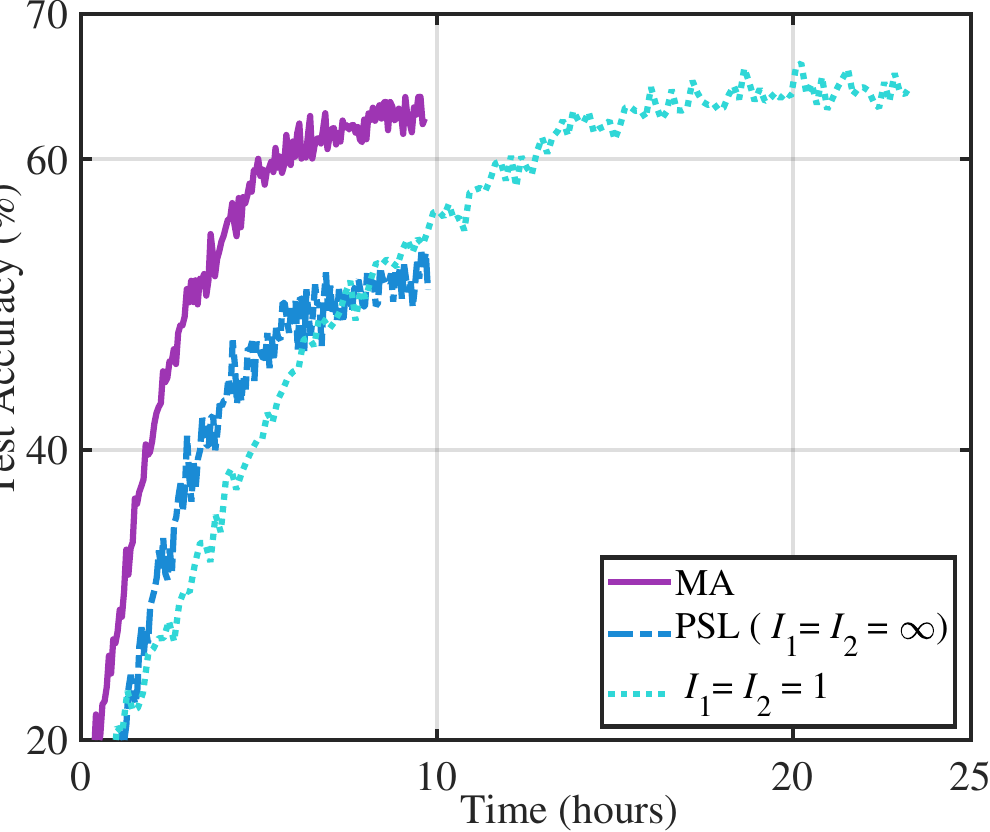}
    \label{sfig:cifar_non_iid_cut_aba}
}
    \caption{The ablation experiments for MA strategy on the CIFAR-10 dataset under IID and non-IID setting with $L_1=3$ and $L_2=8$, {where parallel split learning (PSL)~\cite{lin2023split} is a special case of SFL when $I_1= I_2 = \infty$.}}
    \label{fig:cifar_cut_aba}
\end{figure}

\begin{figure}[t]
    \setlength\abovecaptionskip{6pt}
     \setlength\subfigcapskip{0pt}
     \centering
\subfigure[CIFAR-10 under IID setting.]{
\includegraphics[width=.435\columnwidth]{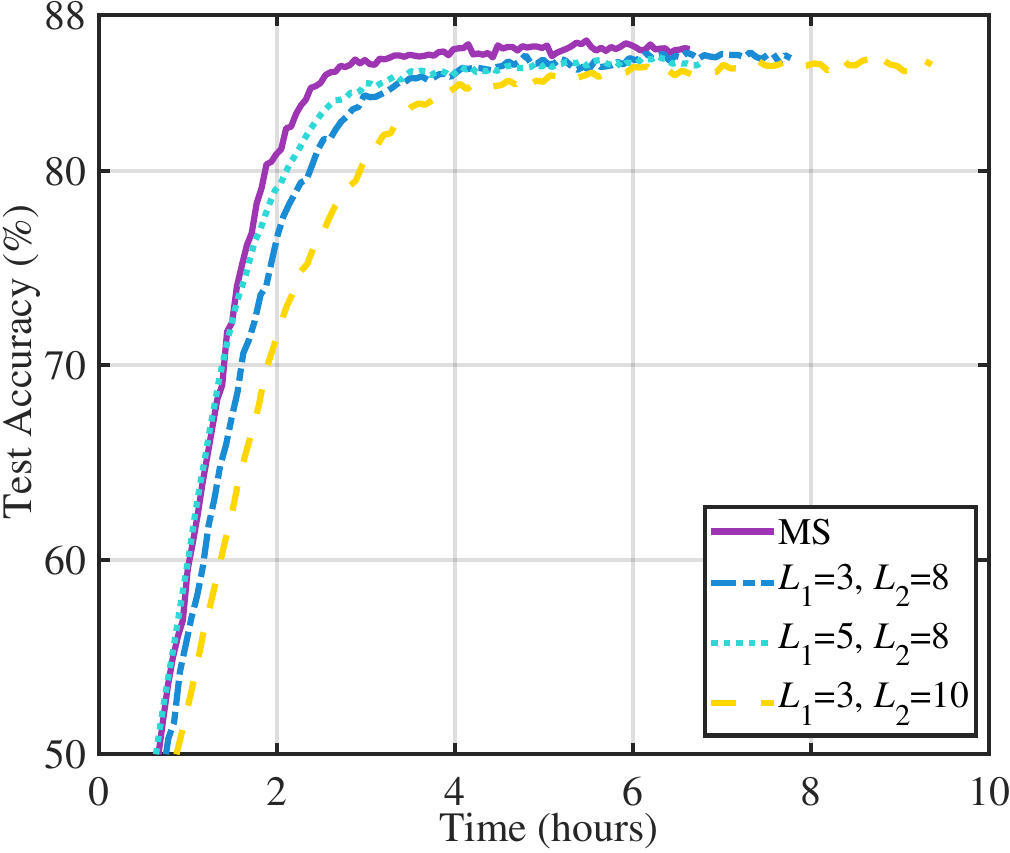}
    \label{sfig:cifar_iid_I_15_different_cut}
}
\subfigure[CIFAR-10 under non-IID setting.]{
    \includegraphics[width=.435\columnwidth]{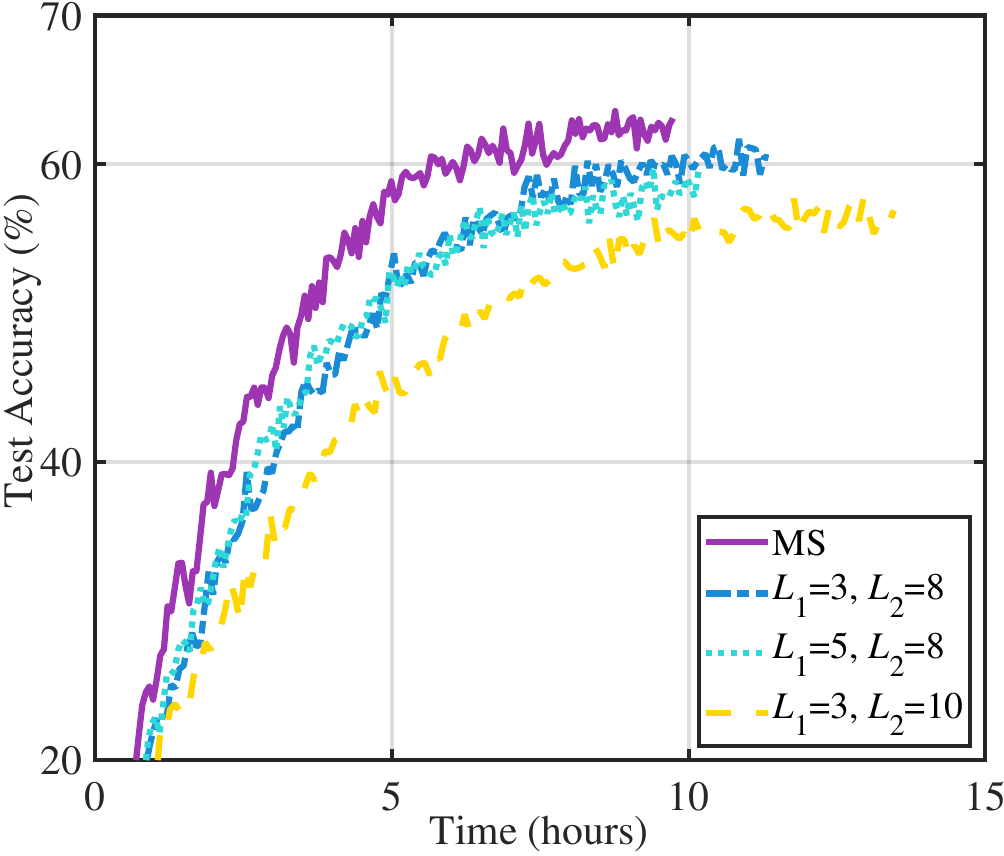}
    \label{sfig:cifar_non_iid_I_15_different_cut}
}
    \caption{The ablation experiments for MS scheme on the CIFAR-10 dataset under IID and non-IID setting with $I_1=140$ and $I_2=20$.}
    \label{fig:I_15_different_cut}
\end{figure}

Fig.~\ref{fig:cifar_cut_aba} shows the impact of MA interval on training performance for the CIFAR-10 dataset. PSL~\cite{lin2023split} showcases the slowest convergence rate and lowest convergence accuracy.  This performance degradation is primarily due to the lack of sub-model MA for all tiers except the $M$-th tier. This absence prevents the sub-models from capturing global information, hindering effective generalization on the global dataset and consequently leading to poor overall performance. In contrast, our proposed sub-model MA strategy significantly accelerates model convergence while achieving an accuracy comparable to that of HSFL with $I_1=I_2=1$ (equivalent to centralized learning).
In IID and non-IID settings, the proposed MA strategy outperforms PSL by achieving $3.4 \%$ and $11.4 \%$ higher convergence accuracy and expediting the model convergence by approximately $23.3 \%$ and $10.1 \%$, respectively. 
The key to these improvements lies in the MA strategy's ability to dynamically adjust the MA interval to achieve the minimum communication-computing latency required for model convergence. This dynamic adjustment ensures that sub-model aggregations occur often enough to maintain model performance but not so frequently as to cause unnecessary communication costs. Therefore, the effectiveness of the proposed MA scheme is demonstrated.

Fig.~\ref{fig:I_15_different_cut} presents the impact of MS on training performance for the CIFAR-10 dataset. It can be seen that the convergence accuracy and speed decline as layers of sub-models with shorter aggregation intervals decrease (e.g., increasing $L_1$ and $L_2$ leads to a reduction in the layers of sub-models located in second and third tiers, respectively), consistent with the derived convergence boundary derived in Eqn.~\eqref{convergence_bound}. The underlying reason for this performance degradation is that reducing the layers of sub-models with shorter aggregation intervals directly lowers the overall update frequency of the entire model, thus leading to a deterioration in generalization performance.
Moreover, the impact of MS on convergence accuracy and speed is more significant under non-IID than IID settings. In the non-IID setting, there are significant disparities in the local data distributions across edge devices. These discrepancies render MA more sensitive because the local dataset may not adequately represent the global data distribution. The comparisons between MS and the other five benchmarks demonstrate that the proposed MS strategy can expedite model convergence while guaranteeing training performance. 

\section{Conclusions}\label{conclu}

In this paper, we propose a novel hierarchical SFL framework, named HSFL, to minimize the training latency of HSFL for achieving target learning performance in a resource-constrained computing system. To guide the system optimization, we first derive the convergence bound of HSFL to quantify the impact of MS and MA on training convergence. Then, Following this, we formulate a joint optimization problem for MA and MS based on the derived convergence bound and develop efficient algorithms to solve it. Extensive simulation results demonstrate that our proposed HSFL framework achieves the target accuracy with significantly less time budget compared to benchmarks, showcasing the effectiveness of the tailored MA and MS strategy.


\appendix


\section*{{A. Proof of Lemma 1}} \label{aa}
For training round $t\geq 1$, we consider the largest $t_{0} \leq  t$ that satisfies $t_0 \bmod I = 0$ (Note that such $t_{0}$ must exist and $t - t_{0} \leq I$). Recalling the Eqn.~\eqref{stage_5_2} for updating the model weights and Eqn.~\eqref{h_c_define} for sub-model aggregation, we can derive 
\begin{align*}
{\bf{w}}_m^{j,t} = {\bf{w}}_{m,n}^{{t_0}} - \gamma \sum\limits_{\tau  = {t_0} + 1}^t {\frac{1}{{N_m^j}}} \sum\limits_{n \in {\cal N}_m^j} {{\bf{g}}_{m,n}^\tau } 
\end{align*}
and 
\begin{align*}
{\bf{\overline w}}_m^t = {\bf{w}}_{m,n}^{{t_0}} - \gamma \sum\limits_{\tau  = {t_0} + 1}^t {\frac{1}{N}} \sum\limits_{j = 1}^{{J_m}} {\sum\limits_{n \in {\cal N}_m^j} {{\bf{g}}_{m,n}^\tau } }. 
\end{align*}
It is worth noting that $\Vert {\bf{\overline w}}_m^t - {\mathbf{w}}^{t}_{m,n} \Vert^{2}=0$ for $I_m=1$. Thus, we have
{
\begin{align*}
&\mathbb{E} [\Vert {\bf{\overline w}}_m^t - {\mathbf{w}}^{t}_{m,n} \Vert^{2}] \\\overset{(a)}=&  {\mathbbm{1}}_{\{I_m > 1\}} \mathbb{E} [\sum\limits_{j = 1}^{{J_m}} {\frac{{N_m^j}}{N}\Vert{\bf{\overline w}}_m^t - {\bf{w}}_m^{j,t}{\Vert^2}}  ] \\=& {\mathbbm{1}}_{\{I_m > 1\}}  \mathbb{E} [\sum\limits_{j = 1}^{{J_m}} \frac{{N_m^j}}{N} \Vert\gamma \sum\limits_{\tau  = {t_0} + 1}^t {\frac{1}{N}} \sum\limits_{j = 1}^{{J_m}} {\sum\limits_{n \in {\cal N}_m^j} {{\bf{g}}_{m,n}^\tau } }  \\ - & \gamma\sum\limits_{\tau  = {t_0} + 1}^t { \frac{1}{{N_m^j}}} \sum\limits_{n \in {\cal N}_m^j} {{\bf{g}}_{m,n}^\tau } {\Vert^2} ]\\
=& {\mathbbm{1}}_{\{I_m > 1\}} \gamma^{2} \mathbb{E} [\sum\limits_{j = 1}^{{J_m}} \frac{{N_m^j}}{N} \Vert  \sum\limits_{\tau  = {t_0} + 1}^t {\frac{1}{N}} \sum\limits_{j = 1}^{{J_m}} {\sum\limits_{n \in {\cal N}_m^j} {{\bf{g}}_{m,n}^\tau } }  \\ - & \sum\limits_{\tau  = {t_0} + 1}^t { \frac{1}{{N_m^j}}} \sum\limits_{n \in {\cal N}_m^j} {{\bf{g}}_{m,n}^\tau } {\Vert^2} ]\\
\overset{(b)}{\leq}&  {\mathbbm{1}}_{\{I_m > 1\}} 2\gamma^{2} \mathbb{E} [\sum\limits_{j = 1}^{{J_m}} \frac{{N_m^j}}{N}[\Vert\sum\limits_{\tau  = {t_0} + 1}^t {\frac{1}{N}} \sum\limits_{j = 1}^{{J_m}} {\sum\limits_{n \in {\cal N}_m^j} {{\bf{g}}_{m,n}^\tau } } {\Vert^2}  \\ + &  \Vert\sum\limits_{\tau  = {t_0} + 1}^t {\frac{1}{{N_m^j}}} \sum\limits_{n \in {\cal N}_m^j} {{\bf{g}}_{m,n}^\tau } {\Vert^2} ] ]\\
\overset{(c)}{\leq}& {\mathbbm{1}}_{\{I_m > 1\}} 2\gamma^{2} (t-t_{0}) \mathbb{E} [\sum\limits_{j = 1}^{{J_m}} \frac{{N_m^j}}{N}[ \sum\limits_{\tau  = {t_0} + 1}^t \Vert{\frac{1}{N}} \sum\limits_{j = 1}^{{J_m}} {\sum\limits_{n \in {\cal N}_m^j} {{\bf{g}}_{m,n}^\tau } } {\Vert^2} \\+& 
\sum\limits_{\tau  = {t_0} + 1}^t \Vert{\frac{1}{{N_m^j}}} \sum\limits_{n \in {\cal N}_m^j} {{\bf{g}}_{m,n}^\tau } {\Vert^2} ]]\\
\overset{(d)}{\leq}& {\mathbbm{1}}_{\{I_m > 1\}} 2\gamma^{2} (t-t_{0}) \mathbb{E} [\sum\limits_{j = 1}^{{J_m}} \frac{{N_m^j}}{N}[ \sum\limits_{\tau  = {t_0} + 1}^t \!\!({\frac{1}{N}} \sum\limits_{j = 1}^{{J_m}} {\sum\limits_{n \in {\cal N}_m^j} \!\!{\Vert{\bf{g}}_{m,n}^\tau } } {\Vert^2}) \\+&
\sum\limits_{\tau  = {t_0} + 1}^t ({\frac{1}{{N_m^j}}} \sum\limits_{n \in {\cal N}_m^j} \Vert{{\bf{g}}_{m,n}^\tau } {\Vert^2}) ]]\\
 \overset{(e)}{\leq} & {\mathbbm{1}}_{\{I_m > 1\}} 4\gamma^{2} I_m^{2} \sum\limits_{l = L_{m-1}+1}^{L_{m-1}+L_m}  {G _l^2},
\end{align*}
}
where ${\mathbbm{1}_{\{ {\rm{\cdot}}\} }}$ represents the indicator function, equal to 1 if the condition ${\{ {\rm{\cdot}}\} }$ holds, and 0 otherwise; (a) holds because multiple sub-models in each computing entity are aggregated every training round; (b)-(d) follows by using the inequality $\Vert \sum_{i=1}^{n} \mathbf{z}_{i}\Vert^{2} \leq n \sum_{i=1}^{n} \Vert \mathbf{z}_{i}\Vert^{2}$ for any vectors $\mathbf{z}_{i}$ and any positive integer $n$ (using $n=2$ in (b), $n=t-t_0$ in (c), and $n=N$ in (d)); and (e) follows from {\bf Assumption \ref{asp:2}}.

\section*{{B. Proof of Theorem 1}} \label{bb}


We consider training round $t\geq 1$. By the smoothness of loss function $f\left(  \cdot  \right)$, we have
{ \begin{align}\label{eq:equal_1_total}
\mathbb{E}[ {f({{\bf{\overline w}}^t})} ] \le& \mathbb{E}[ {f({{\bf{\overline w}}^{t - 1}})} ] + \mathbb{E}[ {\langle {\nabla _{\bf{w}}}f({{\bf{\overline w}}^{t - 1}}),{{\bf{\overline w}}^t} - {{\bf{\overline w}}^{t - 1}}\rangle } ]{\rm{  }} \nonumber \\~+ &\frac{\beta }{2}\mathbb{E}[ {{{\| {{{\bf{\overline w}}^t} - {{\bf{\overline w}}^{t - 1}}} \|}^2}} ].
\end{align}}
Note that
{\begin{align}\label{eq:w_decouple}
 &\mathbb{E}[{{{\| {{{\bf{\overline w}}^t} - {{\bf{\overline w}}^{t - 1}}} \|}^2}}]\nonumber\\
 =&\mathbb{E}[ {{{\| { [{\bf{\overline w}}^t_1, {\bf{\overline w}}^t_2, ..., {\bf{\overline w}}^t_M  ] - [{\bf{\overline w}}^{t-1}_1, {\bf{\overline w}}^{t-1}_2, ..., {\bf{\overline w}}^{t-1}_M  ]} \|}^2}} ]\nonumber\\ =&\mathbb{E}[ {{{\| {[ {{\bf{\overline w}}^{t}_1 - {\bf{\overline w}}^{t-1}_1; {\bf{\overline w}}^{t}_2 - {\bf{\overline w}}^{t-1}_2; ...;{\bf{\overline w}}^{t}_M - {\bf{\overline w}}^{t-1}_M} ]} \|}^2}} ]\nonumber\\=&\sum\limits_{m = 1}^M \mathbb{E}[{{{\| {{{\bf{\overline w}}_m^t} - {{\bf{\overline w}}_m^{t - 1}}} \|}^2}}],
\end{align}}%
where 
$\mathbb{E}[{{{\| {{{\bf{\overline w}}_m^t} - {{\bf{\overline w}}_m^{t - 1}}} \|}^2}}]$ can be bounded as 
{
\begin{align}\label{eq:wc_squre}
 &\mathbb{E}[ {{{\| { {{{\bf{\overline w}}_m^t}\! - {{\bf{\overline w}}_m^{t - 1}}} } \|}^2}} ] \overset{(a)}{=}{\gamma ^2}\mathbb{E}[|| {\frac{1}{N}\sum\limits_{j = 1}^{{J_m}} {\sum\limits_{n \in {\cal N}_m^j} {{\bf{g}}_{m,n}^\tau } } } |{|^2}] \nonumber\\ 
 \overset{(b)}{=}& {\gamma ^2}\mathbb{E}[||\frac{1}{N}\sum\limits_{j = 1}^{{J_m}} {\sum\limits_{n \in {\cal N}_m^j} {({\bf{g}}_{m,n}^\tau } }  - {\nabla _{{{\bf{w}}_m}}}{f_n}({\bf{w}}_{m,n}^{t - 1})) |{|^2}]\nonumber\\ 
 +& {\gamma ^2}\mathbb{E}[||\frac{1}{N}\sum\limits_{j = 1}^{{J_m}} {\sum\limits_{n \in {\cal N}_m^j} {{\nabla _{{{\bf{w}}_m}}}{f_n}\left( {{\bf{w}}_{m,n}^{t - 1}} \right)} }  |{|^2}] \nonumber\\
\overset{(c)}{=}& \frac{{{\gamma ^2}}}{{{N^2}}}\sum\limits_{n = 1}^N  \mathbb{E}[||{({\bf{g}}_{m,n}^\tau } - {\nabla _{{{\bf{w}}_m}}}{f_n}\left( {{\bf{w}}_{m,n}^{t - 1}} \right)|{|^2}]\nonumber\\
+& {\gamma ^2}\mathbb{E}[||\frac{1}{N}\sum\limits_{n = 1}^N {{\nabla _{{{\bf{w}}_m}}}{f_n}\left( {{\bf{w}}_{m,n}^{t - 1}} \right)} |{|^2}] \nonumber\\
\overset{(d)}{\leq }& \frac{{ \sum\limits_{l = L_{m-1}+1}^{L_{m-1}+L_{m}} \!\!\!\!\!{\gamma ^2}{\sigma _l^2}}}{N} +{\gamma ^2}\mathbb{E}[||\frac{1}{N}\!\!\sum\limits_{n = 1}^N {{\nabla _{{{\bf{ w}}_m}}}{f_n}\!\!\left( {{\bf{w}}_{m,n}^{t - 1}} \right)} |{|^2}],
\end{align}}%
where (a) follows from Eqn.~\eqref{stage_5_2} and Eqn.~\eqref{h_c_define}; (b) follows by observing that $\mathbb{E}[\mathbf{g}_{m,n}^{t}] = \nabla_{{\bf{w}}_m} f_{n}({\mathbf{w}}_{m,n}^{t-1})$ and applying the equality $\mathbb{E}[\Vert \mathbf{z} \Vert^{2}] = \mathbb{E} [ \Vert \mathbf{\mathbf{z}} - \mathbb{E}[\mathbf{z}]\Vert^{2}] + \Vert\mathbb{E}[\mathbf{z}] \Vert^{2}$ that holds for any random vector $\mathbf{z}$; (c) follows because each ${{\bf{g}}_{m,n}^\tau } - \nabla_{{\bf{w}}_n} f_{n}(\mathbf{w}_{m,n}^{t-1})$ has zero mean and is independent across edge devices; and (d) follows from {\bf Assumption 2}.

Substituting Eqn.~\eqref{eq:wc_squre} into Eqn.~\eqref{eq:w_decouple} yields
{\begin{align}\label{eq:w-difference-squre}
&\mathbb{E}[\|{{\bf{\overline w}}^t} - {{\bf{\overline w}}^{t - 1}}\|{^2}] \nonumber \\\le& \frac{{{\gamma ^2}\sum\limits_{l = 1}^{L} {\sigma _l^2}}}{N} + {\gamma ^2}\sum\limits_{m= 1}^{M} \mathbb{E}[||\frac{1}{N}\sum\limits_{n = 1}^N {{\nabla _{{{\bf{w}}_m}}}{f_n}\left( {{\bf{w}}_{m,n}^{t - 1}} \right)} |{|^2}] \nonumber \\
=&\frac{{{\gamma ^2}\sum\limits_{l = 1}^{L} {\sigma _l^2}}}{N} + {\gamma ^2}\mathbb{E}[||\frac{1}{N}\sum\limits_{n = 1}^N {{\nabla _{\bf{w}}}{f_n}\left( {{\bf{w}}_n^{t - 1}} \right)} |{|^2}].
\end{align}}

We further note that 
{\begin{align} \label{eq:inner_product_w}
&\mathbb{E}[\langle \nabla_{\bf{ w}} f({\mathbf{\overline w}}^{t-1}), {\mathbf{\overline w}}^{t} - {\mathbf{\overline w}}^{t-1}\rangle] \nonumber\\
\overset{(a)}{=}& -\gamma \mathbb{E} [\langle \nabla_{\bf{w}} f({\mathbf{\overline w}}^{t-1}), \frac{1}{N} \sum_{n=1}^{N} \mathbf{g}_{n}^{t}\rangle] \nonumber \\
\overset{(b)}{=}& -\gamma \mathbb{E}[\langle {\nabla_{\bf{w}}}f({{\bf{\overline w}}^{t - 1}}),\frac{1}{N}\sum\limits_{n = 1}^N \nabla_{\bf{w}}  {f_n}({\bf{w}}_n^{t - 1})\rangle ]\nonumber \\
\overset{(c)}=&  - \frac{\gamma }{2}\mathbb{E}[||{\nabla _{\bf{w}}}f({{\bf{\overline w}}^{t - 1}})|{|^2} + ||\frac{1}{N}\sum\limits_{n = 1}^N {{\nabla _{\bf{w}}}} {f_n}({\bf{w}}_{n}^{t - 1})|{|^2} \nonumber \\
&- ||{\nabla _{{{\bf{w}}}}}f({{\bf{\overline w}}^{t - 1}}) - \frac{1}{N}\sum\limits_{n = 1}^N {{\nabla _{{{\bf{w}}}}}} {f_n}({\bf{w}}_{n}^{t - 1})|{|^2}],
\end{align}}%
where ${\bf{g}}_{n}^t=[{\bf{g}}_{1,n}^t; {\bf{g}}_{2,n}^t; ...; {\bf{g}}_{M,n}^t]$ represents the gradient of client $n$, (a) follows from ${{\bf{\overline w}}^t} = [{\bf{\overline w}}^t_1; {\bf{\overline w}}^t_2; ...; {\bf{\overline w}}^t_M ]$, Eqn.~\eqref{stage_5_2}, and Eqn.~\eqref{h_c_define};
(c) follows from the  identity $\langle \mathbf{z}_{1}, \mathbf{z}_{2}\rangle = \frac{1}{2} \big( \Vert \mathbf{z}_{1}\Vert^{2} + \Vert \mathbf{z}_{2}\Vert^{2} - \Vert \mathbf{z}_{1} - \mathbf{z}_{2}\Vert^{2} \big)$ for any two vectors $\mathbf{z}_{1}, \mathbf{z}_{2}$ of the same length; (b) follows from 
{\begin{align*}
	&\mathbb{E}[\langle \nabla_{\bf{w}} f({\mathbf{\overline w}}^{t-1}), \frac{1}{N} \sum_{n=1}^{N} \mathbf{g}_{n}^{t}\rangle] \\
	=& \mathbb{E}[\mathbb{E}[\langle \nabla_{\bf{w}} f({\mathbf{\overline w}}^{t-1}), \frac{1}{N} \sum_{n=1}^{N} \mathbf{g}_{n}^{t}\rangle | \boldsymbol{\xi}^{[t-1]}]] \\
	=& \mathbb{E}[\langle \nabla_{\bf{w}} f({\mathbf{\overline w}}^{t-1}), \frac{1}{N} \sum_{n=1}^{N} \mathbb{E}[\mathbf{g}_{n}^{t}| \boldsymbol{\xi}^{[t-1]}]\rangle ]\\
	 =& \mathbb{E}[\langle \nabla_{\bf{w}} f({\mathbf{\overline w}}^{t-1}), \frac{1}{N} \sum_{n=1}^{N} \nabla_{\bf{w}} f_{n}(\mathbf{w}_{n}^{t-1})\rangle ],
	\end{align*}}%
where the first equality follows by the law of expectations, the second equality follows because ${\mathbf{\overline w}}^{t-1}$ is determined by $\boldsymbol{\xi}^{[t-1]}= [\boldsymbol{\xi}^{1}, \ldots, \boldsymbol{\xi}^{t-1}]$ and the third equality follows from $\mathbb{E}[\mathbf{g}_{n}^{t} | \boldsymbol{\xi}^{[t-1]}] = \mathbb{E}[\nabla F_{n}(\mathbf{w}_{n}^{t-1};\xi^{t}_{n}) | \boldsymbol{\xi}^{[t-1]}] = \nabla f_{n}(\mathbf{w}_{n}^{t-1})$.

Substituting Eqn.~\eqref{eq:w-difference-squre} and Eqn.~\eqref{eq:inner_product_w} into Eqn.~\eqref{eq:equal_1_total}, we have  
{\begin{align} \label{eq:11}
&\mathbb{E}[f({\mathbf{\overline w}}^{t})] \nonumber\\
\leq &\mathbb{E}[f({\mathbf{\overline w}}^{t-1})] - \frac{\gamma - \gamma^{2}\beta}{2} \mathbb{E} [\Vert \frac{1}{N} \sum_{n=1}^{N} \nabla_{\bf{w}} f_{n} (\mathbf{w}_{n}^{t-1})\Vert^{2}] \nonumber \\
 &- \frac{\gamma}{2} \mathbb{E}[\Vert \nabla_{\bf{w}} f({\mathbf{\overline w}}^{t-1})\Vert^{2}+ \frac{\beta\gamma^{2}\sum\limits_{l = 1}^{L} {\sigma _l^2}}{2N}  \nonumber\\ &+ \frac{\gamma}{2}\mathbb{E}[||{\nabla _{{{\bf{w}}}}}f({{\bf{\overline w}}^{t - 1}}) -\frac{1}{N}\sum\limits_{n = 1}^N {{\nabla _{{{\bf{w}}}}}} {f_n}({\bf{w}}_n^{t - 1})|{|^2}]\nonumber \\
\overset{(a)}{\leq}& \mathbb{E}[f({\mathbf{\overline w}}^{t-1})] - \frac{\gamma}{2} \mathbb{E}[\Vert \nabla_{\bf{w}} f({\mathbf{\overline w}}^{t-1})\Vert^{2}+ \frac{\beta\gamma^{2}\sum\limits_{l = 1}^{L} {\sigma _l^2}}{2N}  \nonumber\\ &+ \frac{\gamma}{2}\sum\limits_{m = 1}^M \mathbb{E}[||{\nabla _{{{\bf{w}}_m}}}f({{\bf{\overline w}}_{m}^{t - 1}}) -\frac{1}{N}\sum\limits_{n = 1}^N {{\nabla _{{{\bf{w}}_m}}}} {f_n}({\bf{w}}_{m,n}^{t - 1})|{|^2}]\nonumber \\
\overset{(b)}{\leq} &\mathbb{E}[f({\mathbf{\overline w}}^{t-1})] - \frac{\gamma}{2} \mathbb{E}[\Vert \nabla_{\bf{w}} f({\mathbf{\overline w}}^{t-1})\Vert^{2}]+ \frac{\beta\gamma^{2}\sum\limits_{l = 1}^{L} {\sigma _l^2}}{2N}  \nonumber\\ &+ 2\beta^2\gamma^{3} \sum\limits_{m = 1}^{M - 1} {\bigg( {{\mathbbm{1}}_{\{I_m > 1\}} I_m^2\! \sum\limits_{l = {L_{m - 1}} + 1}^{{L_{m - 1}} + {L_m}}  \!\!G_l^2} \bigg)},
\end{align}
}%
where (a) follows from $0 < \gamma \leq \frac{1}{\beta}$ and (b) holds because of the following inequality~\eqref{difference_wc} and~\eqref{difference_ws}
{\begin{align}\label{difference_wc}
&\mathbb{E}[ \Vert \nabla_{{{\bf{w}}_m}} f({\mathbf{\overline w}}_m^{t-1}) - \frac{1}{N} \sum_{n=1}^{N} \nabla_{{{\bf{w}}_m}} f_{n} (\mathbf{w}_{m,n}^{t-1})\Vert^{2}] \nonumber \\
 =& \mathbb{E} [ \Vert \frac{1}{N} \sum_{n=1}^{N}\nabla_{{{\bf{w}}_m}} f_{n}({\mathbf{\overline w}}_m^{t-1}) - \frac{1}{N} \sum_{n=1}^{N} \nabla_{{{\bf{w}}_m}} f_{n} (\mathbf{w}_{m,n}^{t-1})\Vert^{2}] \nonumber \\
=& \frac{1}{N^{2}} \mathbb{E} [\Vert \sum_{n=1}^{N} \big( \nabla_{{{\bf{w}}_m}} f_{n}({\mathbf{\overline w}}_m^{t-1}) - \nabla_{{\bf w}_m} f_{n} (\mathbf{w}_{m,n}^{t-1}) \big)\Vert^{2}] \nonumber \\
\leq& \frac{1}{N} \mathbb{E} [ \sum_{n=1}^{N} \Vert \nabla_{{{\bf{w}}_m}} f_{n}({\mathbf{\overline w}}_m^{t-1}) - \nabla_{{{\bf{w}}_m}} f_{n} (\mathbf{w}_{m,n}^{t-1})\Vert ^{2}] \nonumber \\
\leq& \beta^2\frac{1}{N} \sum_{n=1}^{N}\mathbb{E}[ \Vert {\mathbf{\overline w}}_m^{t-1} - \mathbf{w}_{m,n}^{t-1}\Vert^{2}] \nonumber \\
\leq&  {\mathbbm{1}}_{\{I_m > 1\}} 4\beta^2\gamma^{2} I_m^{2}\!\! \sum\limits_{l = L_{m-1}+1}^{L_{m-1}+L_m}  {G _l^2},
\end{align}
}%
where the first inequality follows by using $\Vert \sum_{i=1}^{N} \mathbf{z}_{i}\Vert^{2} \leq N \sum_{i=1}^{N} \Vert \mathbf{z}_{i}\Vert^{2}$ for any vectors $\mathbf{z}_{i}$; the second inequality follows from the smoothness of each local loss function $f_{n}$ by {\bf Assumption \ref{asp:1}}; and the third inequality follows from {\bf Lemma \ref{lm:diff-avg-per-node}}. Moreover, we have
{ \begin{align}\label{difference_ws}
&\mathbb{E}[ \Vert \nabla_{{{\bf{w}}_M}} f({\mathbf{\overline w}}_M^{t-1}) - \frac{1}{N} \sum_{n=1}^{N} \nabla_{{{\bf{w}}_M}} f_{n} (\mathbf{w}_{M,n}^{t-1})\Vert^{2}] \nonumber \\
\leq& \beta^2\frac{1}{N} \sum_{n=1}^{N}\mathbb{E}[ \Vert {\mathbf{\overline w}}_M^{t-1} - \mathbf{w}_{M,n}^{t-1}\Vert^{2}] \overset{(a)}{=}0, 
\end{align}
where (a) holds because sub-models of $M$-th tier are aggregated in each training round (i.e., $I_M = 1$). Therefore, at any training round $t$, the $M$-th tier sub-model of each client is the aggregated version of sub-models.

Dividing the both sides of Eqn.~\eqref{eq:11} by $\frac{\gamma}{2}$ and rearranging terms yields
{ \begin{align}\label{}
&\mathbb{E}\left [\Vert \nabla_{\bf{w}} f({\mathbf{\overline w}}^{t-1})\Vert^{2}\right] 
\leq\frac{2}{\gamma} \left(\mathbb{E}\left[f({\mathbf{\overline w}}^{t-1})\right]  - \mathbb{E}\left[f({\mathbf{\overline w}}^{t})\right]\right) \nonumber \\+& \frac{\beta\gamma \sum\limits_{l = 1}^{L} {\sigma_l^2} }{N}  \nonumber + 4\beta^2\gamma^{2} \sum\limits_{m = 1}^{M - 1} {\bigg( {{\mathbbm{1}}_{\{I_m > 1\}} I_m^2 \sum\limits_{l = {L_{m - 1}} + 1}^{{L_{m - 1}} + {L_m}}  G_l^2} \bigg)}. 
\end{align}
}%
Summing over $t\in\{1,\ldots, R\}$ and dividing both sides by $R$ yields
{\scriptsize \begin{align*}
&\frac{1}{R} \sum_{t=1}^{R} \mathbb{E}\left [\Vert \nabla_{\bf{w}} f({\mathbf{\overline w}}^{t-1})\Vert^{2}\right] \\
\leq& \frac{2}{\gamma R} \!\!\left(\!f({\mathbf{\overline w}}^{0}) \!-\! \mathbb{E}\!\!\left[f({\mathbf{\overline w}}^{R})\right]\!\right) \!\!+\!\frac{\beta\gamma \!\sum\limits_{l = 1}^{L} \!\!{\sigma_l^2} }{N}  \nonumber \!+\! 4\beta^2\gamma^{2}\!\!\sum\limits_{m = 1}^{M - 1} \!\!\!{\bigg( {{\mathbbm{1}}_{\{I_m > 1\}} I_m^2\!\!\!\!\!\! \sum\limits_{l = {L_{m - 1}} + 1}^{{L_{m - 1}} + {L_m}}  \!\!\!\!\!G_l^2} \bigg)} \\
\overset{(a)}{\leq} &  \frac{2}{\gamma R} \left(f({\mathbf{w}}^{0}) -f^{\ast}\right) +\!\frac{\beta\gamma \sum\limits_{l = 1}^{L} {\sigma_l^2} }{N}  \nonumber \!+\! 4\beta^2\gamma^{2}\! \sum\limits_{m = 1}^{M - 1} \!\!{\bigg(\! {{\mathbbm{1}}_{\{I_m > 1\}} I_m^2 \!\!\sum\limits_{l = {L_{m - 1}} + 1}^{{L_{m - 1}} + {L_m}}  \!\!G_l^2} \bigg)},
\end{align*}}
where (a) follows because $f^{\ast}$ is the minimum value of problem {\eqref{minimiaze_loss_function}}.

\section*{{C. Proof of Proposition 1}} \label{bb}

Assuming $I_m>1, \forall m \in {\mathcal{M}}\setminus{\{{M}\}}$, we perform the functional analysis for the objective function of problem~\eqref{subproblem_1}. We denote $\Theta '({\bf{I}})$ as $\Theta '({\bf{I}}) = \frac{{2\vartheta }}{\gamma }\frac{{a{e_m}{I_m} + \sum\limits_{m \in {\cal M}''} {{b_m}{e_m}} }}{{{e_m}{I_m}(c - 4{\beta ^2}{\gamma ^2}(\sum\limits_{m \in {\cal M}''}  {d_m}I_m^2))}}$, where $a = {T_1}$, ${b_m} = {T_{m,2}} + {T_{m,3}}$, $c = \varepsilon  - \frac{{\beta \gamma \sum\limits_{l = 1}^L {\sigma _l^2} }}{N}$, ${d_m} = \sum\limits_{l = 1}^L  \left( {{\mu _{m,l}} - {\mu _{m - 1,l}}} \right){\widetilde G}_l^2$ and ${e_m} = \frac{{\mathop \Pi \limits_{m \in {\cal M}''} {I_m}}}{{{I_m}}}$. Without loss of generality, we take the first-order derivative for the arbitrary $m'$-th tier (${m' \in {\cal M}''}$) sub-model MA interval $I_{m'}$ yields
\begin{align}\label{accuracy_cons}
\frac{{\partial \Theta '({\bf{I}})}}{{\partial {I_{m'}}}} = \frac{{2\vartheta }}{\gamma }\frac{{\Xi \left( {\bf I} \right)}}{{{{\big( {{e_{m'}}{I_{m'}}(c - 4{\beta ^2}{\gamma ^2}(\sum\limits_{m \in {\cal M}''}  {d_m}I_m^2))} \big)^2}}}},
\end{align}
where 
{\footnotesize \begin{align}\label{E_I}
&\Xi \!\left( {\bf{I}} \right) \!=\! 8(a{e_{m'}} \!+\!\!\!\!\!\!\!\!\!\!\!\! \sum\limits_{m \in {{\cal M}''\setminus{\{{m'}\}}}}  \!\!\!\!\!\! {\frac{{{b_m}{e_{m'}}}}{{{I_m}}}){e_{m'}}} {\beta ^2}{\gamma ^2}{d_{m'}}I_{m'}^3 \!\!+\!\!12{b_{m'}}e_{m'}^2{\beta ^2}{\gamma ^2}{d_{m'}}I_{m'}^2 \nonumber \\&- {b_{m'}}e_{m'}^2(c - 4{\beta ^2}{\gamma ^2}\!\!\!\!\!\!\!\!\!\!\sum\limits_{m \in {{\cal M}''\setminus{\{{m'}\}}}}  \!\!\!\!\!\!\!\!{d_m}I_m^2).
\end{align}}
Since $\frac{{\partial \Xi \left( {\bf{I}} \right)}}{{\partial {I_{m'}}}} = 24{\beta ^2}{\gamma ^2}{e_{m'}}{d_{m'}}{I_{m'}}(a{e_{m'}}{I_{m'}} + \sum\limits_{m \in {{\cal M}''}} {{b_m}{e_m}} ) > 0$, $\Xi \left( {\bf I} \right)$ is a monotonically increasing function for any $I_{m'}$. Given that ${\left. {\Xi \left( {\bf{I}} \right)} \right|_{{I_{m'}} = 0}} \!=\!  - e_{m'}^2{b_{m'}}(c \!-\! 4{\beta ^2}{\gamma ^2}\!\!\!\!\!\!\!\!\!\sum\limits_{m \in {{\cal M}''\setminus{\{{m'}\}}}} \!\!\!\!\!\!\!\! {d_m}I_m^2) < 0$ and ${\left. {\Xi \left( {\bf{I}} \right)} \right|_{{I_{m'}} \to  + \infty }} =  + \infty  > 0$, there must exist ${{\hat I}_{m'}}$ satisfying ${\left. {\Xi \left( {\bf{I}} \right)} \right|_{{I_{m'}} = {{\hat I}_{m'}}}} = 0$ for ${{\hat I}_{m'}} \in \left( {0, + \infty } \right)$, indicating that the objective function decreases and then increases with respective to $I_{m'}$ and thus reaches a minimum at $I = {\hat I}_{m'}$. Considering $|{\cal M}''|$ sub-model MA intervals, we entend equation~\eqref{E_I} to a system of $|{\cal M}''|$ equations, i.e., $\frac{{\partial \Theta '({\bf{I}})}}{{\partial {\bf{I}}}} = [\frac{{\partial \Theta '({\bf{I}})}}{{\partial {I_{m_1}}}},\frac{{\partial \Theta '({\bf{I}})}}{{\partial {I_{m_2}}}},...,\frac{{\partial \Theta '({\bf{I}})}}{{\partial {I_{m_{|{\cal M}''|}}}}}] = {\bf{0}}$ (${m_1, m_2, ..., m_{|{\cal M}''|} \in {\cal M}''}$). Then, we utilize the classical Newton-Jacobi iterative method~\cite{geng2019pipeline} to obtain the solution ${\bf{\hat I}} = [ {{{\hat I}_{m_1}},{{\hat I}_{m_2}},...,{\hat I}_{m_{|{\cal M}''|}}}]$. According to constraint $\mathrm{C6}$, it is obvious that arbitrary optimal $m'$-th tier sub-model MA interval $I_{m'}^*$ only exists on both sides of ${\hat I}_{m'}$, namely,
{\begin{align}\label{accuracy_cons}
{{\bf{I}}^*} =  {{\mathrm{argmin}} _{\scriptstyle{\bf{I}} \in \{ \left. {I_{m}} \right|\hfill
\scriptstyle{I_{m}} \in \{ \max (\left\lfloor {{{\hat I}_{m}}} \right\rfloor, \left\lceil {{{\hat I}_{m}}} \right\rceil )\} \} \hfill}}\Theta '\!\!\left( {\bf{I}} \right),
\end{align}}
where ${\bf{\hat I}} = \{ {{\hat I}_m}\mid m \in {\cal M}''\} $, $\left\lfloor {\cdot} \right\rfloor$ and $\left\lceil {\cdot} \right\rceil$ represent floor and ceiling operations.

\ifCLASSOPTIONcaptionsoff
  \newpage
\fi



%



\bibliographystyle{IEEEtran}
\bibliography{reference}

\end{document}